\documentclass{article}

\usepackage{microtype}
\usepackage{graphicx}
\usepackage{subcaption}
\usepackage{booktabs} % for professional tables

\usepackage{hyperref}

\usepackage[accepted]{icml2026}

\usepackage{amsmath}
\usepackage{amssymb}
\usepackage{mathtools}
\usepackage{amsthm}

\usepackage{nicefrac}

\definecolor{mygreen}{RGB}{146, 199, 113} 
\definecolor{mypurple}{RGB}{146, 199, 113}
\usepackage[table]{xcolor}
\usepackage{multirow}

\usepackage{enumitem} % 支持自定义 label 格式
\usepackage{wrapfig}

\usepackage[capitalize,noabbrev]{cleveref}

\theoremstyle{plain}
\newtheorem{theorem}{Theorem}[section]

\theoremstyle{definition}

\newtheorem{assumption}[theorem]{Assumption}
\theoremstyle{remark}

\usepackage[textsize=tiny]{todonotes}

\icmltitlerunning{Breaking the Scale Barrier: One-Shot Knowledge Transfer via Frequency Transform}

\begin{document}

\twocolumn[
  \icmltitle{Breaking the Scale Barrier: One-Shot Knowledge Transfer \\ via Frequency Transform}
%标题1：Breaking the Size Barrier: One-Shot Knowledge Transfer via Frequency Transform
%标题2：Extracting Transferable Learngenes for Arbitrary Architectures



  \icmlsetsymbol{equal}{*}
    \icmlsetsymbol{corr}{$\dagger$}
  \begin{icmlauthorlist}
  \icmlauthor{Jianlu Shen}{yyy,sch}
  \icmlauthor{Fu Feng}{yyy,sch}
  \icmlauthor{Yucheng Xie}{yyy,sch}
  \icmlauthor{Jiaqi Lv}{yyy,sch,corr}
  \icmlauthor{Xin Geng}{yyy,sch,corr}
    % \icmlauthor{Firstname1 Lastname1}{equal,yyy}
    % \icmlauthor{Firstname2 Lastname2}{equal,yyy,comp}
    % \icmlauthor{Firstname3 Lastname3}{comp}
  \end{icmlauthorlist}

  \icmlaffiliation{yyy}{
  % Department of XXX, University of YYY, Location, Country 
  School of Computer Science and Engineering, Southeast University, Nanjing, China}
  \icmlaffiliation{sch}{
  % School of ZZZ, Institute of WWW, Location, Country
  Key Laboratory of New Generation Artificial Intelligence Technology and Its Interdisciplinary Applications (Southeast University), Ministry of Education, China
  }

  \icmlcorrespondingauthor{Jiaqi Lv}{jiaqi.lv@seu.edu.cn}
  \icmlcorrespondingauthor{Xin Geng}{xgeng@seu.edu.cn}

  \icmlkeywords{Machine Learning, ICML}

  \vskip 0.3in
]



\printAffiliationsAndNotice{}  % no special notice (required even if empty)

\begin{abstract}
Transferring knowledge by fine-tuning large-scale pre-trained networks has become a standard paradigm for downstream tasks, yet the knowledge of a pre-trained model is tightly coupled with monolithic architecture, which restricts flexible reuse across models of varying scales.
In response to this challenge, recent approaches typically resort to either parameter selection, which fails to capture the interdependent structure of this knowledge, or parameter prediction using generative models that depend on impractical access to large network collections. 
% In this paper, we empirically demonstrate that a model's foundational, task-agnostic knowledge -- its ``learngene" -- is predominantly encoded within the low-frequency components of its weights, and can be efficiently inherited by downstream models.
In this paper, we identify the low-frequency components of model weights as the concrete carrier of foundational, task-agnostic knowledge—its ``learngene"—and validate this by demonstrating its efficient inheritance by downstream models and tasks.
Based on this insight, we propose FRONT (FRequency dOmain kNowledge Transfer), a novel framework that uses the Discrete Cosine Transform (DCT) to isolate the low-frequency ``learngene".
This learngene can be seamlessly adapted to initialize models of arbitrary size via simple truncation or padding, a process that is entirely training-free. 
For enhanced performance, we propose an optional low-cost refinement process that introduces a spectral regularizer to further improve the learngene's transferability.
Extensive experiments demonstrate that FRONT achieves the state-of-the-art performance, accelerates convergence by up to $15\times$ in vision tasks, and reduces training FLOPs by an average of 40.5\% in language tasks. Code is available at https://github.com/LUcy0505/FRONT.
\end{abstract}

\section{Introduction}

The paradigm of fine-tuning large-scale pre-trained models has established itself as the de facto standard in deep learning, facilitating the transfer of general-purpose knowledge from broad datasets to specific downstream tasks and consistently outperforming training from scratch~\citep{kolesnikov2020big, hu2021lora}. While powerful pre-trained models are readily accessible via open-source communities~\citep{wolf2019huggingface,wightman2019pytorch}, their learned knowledge remains tightly coupled with specific, monolithic, and often computationally prohibitive architectures.

\begin{figure*}[t]
  \centering

  \includegraphics[width=\linewidth]{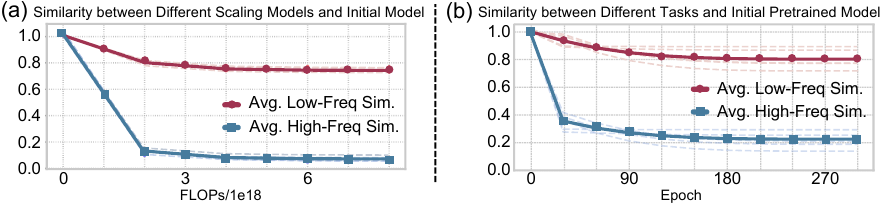}
  \vspace{-0.19in}
  \caption
  {
 (a) Five different-scaled models initialized from the same DeiT are fine-tuned on the same downstream tasks.
   (b) A pre-trained DeiT is fine-tuned on five different downstream tasks. 
   We plot the cosine similarity between the low-frequency components of the fine-tuned weights and their original pre-trained state over the training process. Each dashed line represents an individual fine-tuned model, while the solid line indicates their average.
  }
  \label{fig:jpeg}
  \vspace{-0.18in}
\end{figure*}

To mitigate this rigidity, extensive research has sought to decouple transferable knowledge from its architectural constraints. These approaches broadly fall into two categories. Generative approaches predict parameters by learning latent distributions from model zoos~\citep{knyazev2023can,wang2024neural, soro2024diffusion, schurholt2024towards}. Despite their flexibility, they suffer from a heavy dependency on prohibitively large, homogeneous model collections and, due to high computational overhead, are often restricted to generating partial parameters (e.g., only normalization layers). Conversely, single-source methods extract parameters directly from individual networks~\citep{wang2023learning,xu2023initializing}. However, being confined to the spatial domain, they treat knowledge merely as discrete, local components (e.g., layers and neurons), thereby failing to preserve the global parameter correlations of the source model. 
Fundamentally, both paradigms lack a principled mechanism to isolate the intrinsic constituents of transferable knowledge, inevitably leading to incomplete retention or structural corruption of the learned representations.

The concept of \emph{learngene}~\citep{feng2023genes} is proposed precisely as the theoretical answer to this deficiency. It posits that neural networks internalize a compact form of foundational knowledge, akin to genetic codes, that is not tied to specific architectures or tasks. If such a learngene could be successfully extracted and transferred, any downstream model could inherit a strong foundation and adapt to downstream tasks with extreme efficiency. To materialize this concept, pioneering works have explored various extraction mechanisms. 
Major approaches rely on selecting network fragments~\citep{wang2022learngene,wang2023learngene}, which risks compromising crucial parameter interdependencies and capturing only a fragmented representation of the true learngene. Similarly, other recent approaches involve training an auxiliary model from scratch~\citep{xia2024transformer,feng2024wave}, a costly process that contradicts the vision of a truly universal and efficient initialization. Consequently, a critical gap persists between the conceptual ideal and practical implementation.

To bridge this gap, we shift the focus from direct parameter manipulation to the intrinsic spectral properties of model knowledge. 
While neural weights lack the fixed spatial semantics, unlike images, we posit that the optimization process itself (e.g., SGD)~\citep{robbins1951stochastic, bottou2010large} introduces a strong inductive bias. This optimization pressure drives weights toward structured, low-rank manifolds~\citep{gunasekar2017implicit, huh2021low}, characterized by strong inter-neuron correlations. 
Spectrally, these correlations manifest as \textit{energy concentration}. Consequently, a model's core capability is encoded in this energy distribution rather than superficial spatial indexing.
We validate this hypothesis through two key observations grounded in the Discrete Cosine Transform (DCT).
First, as illustrated in Figure~\ref{fig:jpeg}~(a), we initialized models for various downstream tasks with the same pre-trained DeiT~\citep{touvron2021training} and analyzed their parameter updates in the frequency domain throughout training. 
Our key finding is that low-frequency components remained remarkably stable while high-frequency components underwent significant updates. 
Second, this suggests that low-frequency representations indeed encode the foundational, task-agnostic knowledge that is universally applicable across diverse downstream tasks, as visualized in Figure~\ref{fig:jpeg}~(b). Collectively, these findings suggest that low-frequency components serve as the ideal carrier of the architecture- and task-agnostic knowledge that defines the learngene.

Building on this insight, we introduce FRONT (FRequency dOmain kNowledge Transfer), a novel framework that operationalizes the learngene concept by instantiating it as a low-frequency representation.
FRONT employs the DCT to transform a pre-trained model's weights into the frequency domain, where the low-frequency coefficients are isolated as the learngene. 
A key advantage of FRONT is its flexibility in balancing efficiency and performance. We offer two complementary strategies for learngene acquisition.
For plug-and-play deployment, the learngene can be directly extracted from any off-the-shelf pre-trained model, a process completed in \emph{milliseconds on a CPU}. 
For scenarios demanding enhanced performance, we propose an optional \emph{one-time refinement process}. This refinement introduces a novel spectral regularizer that penalizes high-frequency components during training, encouraging the model to consolidate its knowledge into lower frequencies. Importantly, this refinement can be applied either by training a model from scratch or, more practically, by briefly fine-tuning a pre-existing model within just a few epochs.
Then, the learngene is adapted to a target \emph{architecture of flexible depth or width} by simple truncation or padding before being transformed back into the spatial domain via the IDCT to yield the final initialized weights. 
This approach bridges a critical gap in the learngene's concept and operation by unlocking a previously unavailable capability: training-free, multi-size initialization from any pre-trained model.

Our extensive experiments demonstrate both the state-of-the-art performance and the efficiency of FRONT in initializing models across diverse settings, including vision and language tasks. In the vision domain, models initialized with FRONT achieve the performance of a standard 150-epoch pre-training schedule in only 10 epochs, accelerating convergence by a factor of 15.
This efficiency is mirrored in language tasks, where FRONT reduces the required training FLOPs by an average of 40.5\% across all evaluated architectures compared to training from scratch. Furthermore, visualizations confirm that the extracted learngenes are highly structured, exhibiting patterns consistent with those observed in pre-trained models~\citep{trockman2023mimetic, xu2023initializing}.
By instantiating the learngene as low-frequency components, we achieve a ``one-for-all" parameter initialization that efficiently transfers knowledge from a single model to downstream models of various sizes.

% A key advantage of FRONT is its flexibility. 
% For plug-and-play deployment, the learngene can be directly \emph{extracted from any off-the-shelf} pre-trained model, a process completed in \emph{milliseconds on a CPU}. 
% Alternatively, for seeking enhanced performance, we propose a \emph{one-time refinement process}. 
% This refinement can be applied either by training a model from scratch or, more efficiently, by briefly fine-tuning a pre-existing model within just a few epochs.
% It refines the source model with a novel spectral regularizer that penalizes high-frequency components, encouraging the model to discard task-specific details and preserve a more fundamental knowledge structure.
% \vspace{-0.15in}
\section{Related Work}

\textbf{Model Initialization.}
Effective weight initialization is a cornerstone of deep learning, evolving from early distribution-based methods like Xavier and He initialization~\citep{Glorot_Bengio_2010, Chen_Xie_He_2021} to the current paradigm of leveraging large-scale pre-trained models~\citep{he2016deep, devlin-etal-2019-bert}. While pre-training provides a powerful starting point, it is challenging to efficiently adapt a single pre-trained model to target architectures of mismatched sizes. One line of work involves direct parameter manipulation; for instance, Wt Select~\citep{xu2023initializing} reuses weights by selecting parameter subsets from larger networks, while LiGO~\citep{wang2023learning} applies mathematical scaling for cross-size adaptation. Another direction employs complex generative models, such as GHN-3~\citep{knyazev2023can}, which uses a graph-based hypernetwork to synthesize entirely new parameters. Despite their ingenuity, these methods often suffer from high computational overhead, risk disrupting learned knowledge structures, or introduce parameter inconsistencies that can lead to negative transfer~\citep{feng2024wave}. However, our approach preserves the core transferable knowledge from a source model and flexibly adapts it to various target sizes, enabling effective initialization without costly retraining.

% \textbf{The DCT in Neural Networks.} Renowned for its ``energy compaction'' property~\citep{wallace1991jpeg}, DCT concentrates signal information into low-frequency coefficients and underpins compression standards like JPEG~\citep{raid2014jpeg}. In deep learning, recent works apply DCT to compress model weights by pruning high-frequency components~\citep{ulicny2022harmonic,ulicny2021tensor}. However, these methods treat preserved low-frequency components as static, task-specific artifacts on a single model. In contrast, we identify these coefficients as transferable \emph{learngenes}, that is core, architecture-agnostic knowledge adaptable to diverse tasks.

\textbf{The DCT in Neural Networks.} Renowned for its ``energy compaction'' property~\citep{wallace1991jpeg}, DCT concentrates signal information into low-frequency coefficients and underpins compression standards like JPEG~\citep{raid2014jpeg}. In deep learning, alongside data and knowledge distillation~\nocite{li2024towards,li2025adaptive,kou2026positive,kou2025nips,kou2026fedharmony}, recent works apply DCT to compress model weights by pruning high-frequency components~\citep{ulicny2022harmonic,ulicny2021tensor}. Beyond compression, frequency analysis also reveals model behaviors, such as robustness asymmetry in AI-generated images~\nocite{wang2026ra}. However, these methods treat preserved low-frequency components as static, task-specific artifacts on a single model. In contrast, we identify these coefficients as transferable \emph{learngenes}, that is core, architecture-agnostic knowledge adaptable to diverse tasks.

\textbf{Learngene.}
The learngene paradigm, inspired by biological genetics, proposes transferring compact, knowledge-rich network representation to initialize diverse architectures~\citep{feng2023genes,feng2024wave,xie2025kind,xie2025divcontrol}\nocite{liu2026xpertexpertknowledgetransfer,liuinheriting}.  Some methods, such as HeurLG and AutoLG~\citep{wang2022learngene,wang2023learngene}, typically rely on either heuristic-based selection of fragments, using criteria like gradient stability or representation similarity, or on multi-stage linear expansion like TLEG~\citep{xia2024transformer} and WAVE~\citep{feng2024wave}. However, the latter group suffers from the need for high-cost auxiliary training and relies heavily on strong prior assumptions.
The reliance on such discrete, structurally rigid fragments or expensive retraining contradicts the Learngene vision of a truly universal and efficient initialization paradigm.
In contrast, FRONT directly utilizes full pre-trained parameters, which enables the initialization of models with varying sizes in a single and unified operation.

\vspace{-0.12in}
\section{Methods}
\label{headings}

\subsection{Overview of FRONT}

Our method extends frequency-domain knowledge to the weight spaces—exemplified here with Vision Transformers (ViTs)—through 3D-DCT weight transformations, as shown in Figure~\ref{framework}. Task-agnostic knowledge within the model's parameters is predominantly encoded in the low-frequency components, hereafter referred to as \textbf{\textit{learngenes}}. Notably, the applicability of this approach extends beyond transformer-based models to include other architectures (see details in Appendix~\ref{FRONT Algorithms}), such as Multi-Layer Perceptrons (MLPs) and Convolutional Neural Networks (CNNs).

The learngenes can be extracted using two distinct strategies, detailed in Sections~\ref{sec:direct} and~\ref{sec:regularization}, offering a trade-off between efficiency and performance. \textbf{(1) FRONT{}}: Direct extraction immediately harvests low-frequency coefficients from any off-the-shelf pre-trained model for instantaneous, zero-cost initialization. \textbf{(2) FRONT{\scalebox{1}{\text{+}}}}: Refinement employs a frequency regularizer to generate more effective learngenes, a process that can be implemented from scratch or, more efficiently, through a brief fine-tuning of a pre-existing model, often \emph{in just a few epochs}. Regardless of the chosen strategy, the resulting learngenes serve as a universal blueprint, enabling rapid and flexible initialization of diverse target models in milliseconds on a CPU via IDCT-based reconstruction in Section~\ref{sec:init}.

\begin{figure*}[t]
  \centering
  \includegraphics[width=0.99\linewidth]{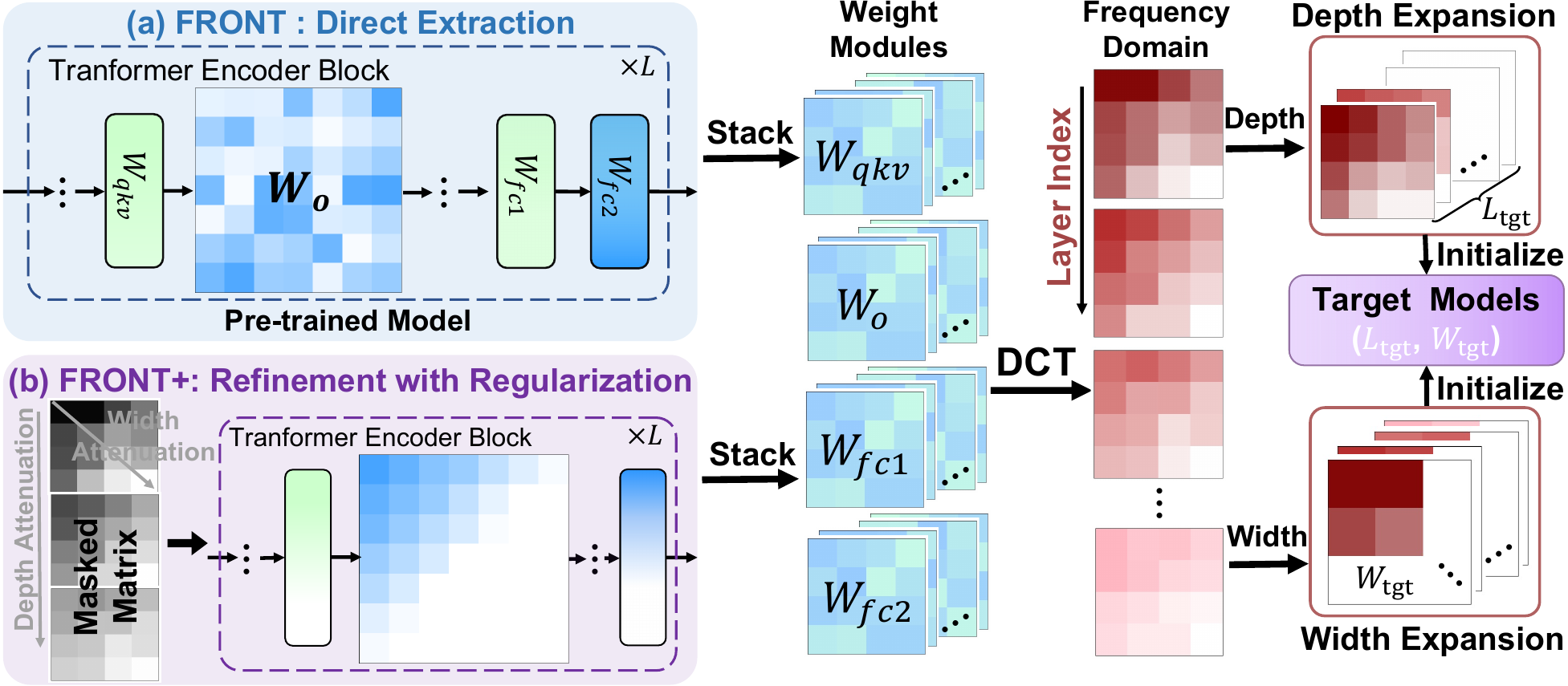}
  \vspace{-0.08in}
  \caption{
  {\textbf{The overview of FRONT and FRONT{\scalebox{1}{\text{+}}}.} The learngenes are obtained from either FRONT via the open available pre-trained models or FRONT{\scalebox{1}{\text{+}}} using source models refined with our frequency regularization. The core process involves transforming weight modules via DCT to extract task-agnostic knowledge. Finally, these learngenes are used to initialize diverse target models through the IDCT, which flexibly accommodate varying architectures via zero-padding or truncation.}
  }
  \label{framework}
  \vspace{-0.17in}
\end{figure*}

\subsection{Preliminary}
\label{prel}
% DCT efficiently converts signals to frequency-domain representations. While 1D/2D-DCT are prevalent in traditional signal/image processing, the 3D-DCT/IDCT variant is adopted here to handle the structured weight matrices in neural networks.
% The transformation formulas for other dimensions and more detailed formulas are shown in the Appendix~\ref{dctandIDCT}.
DCT provides an efficient mechanism for converting signals into frequency-domain representations. Crucially, it is applicable to arbitrary signals, remaining effective \textit{regardless of their inherent smoothness or continuity}. While 1D and 2D-DCT are prevalent in traditional signal and image processing, we adopt the 3D-DCT/IDCT variant to handle the structured weight matrices in neural networks. Detailed formulations are provided in Appendix~\ref{dctandIDCT}.

\paragraph{3D-DCT} For the weight ${x} \in \mathbb{R}^{M \times N \times P}$, the transformed coefficient matrix ${X} \in \mathbb{R}^{M \times N \times P}$ is:
{
\small
\begin{equation}
\begin{split}
X[k,l,q] = \mathcal{D}({x}) &= \alpha(k)\alpha(l)\alpha(q) \sum_{m,n,p} \\
                             &\quad \times x[m,n,p] \, C(m,n,p,k,l,q).
\end{split}
\end{equation}
}

\paragraph{3D-IDCT} The IDCT recovers the original signal:

{
\small
\begin{equation}
\label{eq:3d_IDCT}
\begin{split}
x[m,n,p] = \mathcal{D}^{-1}({X}) &= \sum_{k,l,q} \alpha(k)\alpha(l)\alpha(q) \\
                                 &\quad \times X[k,l,q] \, C(m,n,p,k,l,q).
\end{split}
\end{equation}
}
where the summations are over $m, k \in [0, M-1]$, $n, l \in [0, N-1]$, and $p, q \in [0, P-1]$.

% \subsection{FRONT{}: Direct Extraction}
\subsection{FRONT: Directly Extracted Learngenes}

\label{sec:direct}
The direct extraction strategy operates directly on open available pre-trained models, eliminating the computational overhead of any additional training. For ViTs, weight matrices are collected from all $L$ transformer layers, including the queries, keys, and values ($W_{qkv}^{(1\thicksim L)}$), attention output ($W_{o}^{(1\thicksim L)}$), and two feed-forward layers ($W_{{fc1}}^{(1\thicksim L)}, W_{{fc2}}^{(1\thicksim L)}$), as follows:
\begin{equation}
\Omega = \{W_{qkv}, W_{o}, W_{{fc1}}, W_{{fc2}}\} \subset \mathbb{R}^{L \times d_{\text{in}} \times d_{\text{out}}},
\end{equation}
where $L$ is the total number of transformer layers, and $d_{\text{in}}$, $d_{\text{out}}$ are input and output dimensions. 

Based on the finding that a model's foundational knowledge is concentrated in its low-frequency coefficients shown in Figure~\ref{fig:jpeg}, we extract these \textbf{learngenes} as follows:

% \textbf{DCT Transformation:} Each $W^i \in \Omega$ is transformed into its frequency domain using 3D-DCT:
% % \begin{equation}
% % \label{eq:dct}
% % \mathbf{\Phi}_i = \mathcal{D}(\mathbf{W}_i),
% % \end{equation}
% $\Phi^i = \mathcal{D}(W^i)$,
% where the resulting $\Phi^i$ preserves the dimensions of the original weight.

\textbf{DCT Transformation:} Each weight matrix $W^i \in \Omega$ is transformed into the frequency domain via 3D-DCT:
$\Phi^i = \mathcal{D}(W^i)$,
where the generated spectral representation $\Phi^i$ retains the same dimensions as the original weight $W^i$.

\textbf{Preservation and Truncation:} A binary mask $\mathbf{M}_r$ is applied to preserve only the low-frequency coefficients along all three dimensions (layer, input, output), defined by the frequency ratio $r$:
\begin{equation}
\label{eq:truncation}
\mathcal{G}^i = {M}_r \odot {\Phi}^i, \quad {M}_r[l,m,n] =
\begin{cases}
1 & \text{if } \mathcal{C} \\
0 & \text{otherwise}
\end{cases}
\text{,}
\end{equation}
where $ \odot$ denotes truncating the coefficient values of zero after dot multiplication, and the condition $\mathcal{C}$ is defined as $l \leq \lfloor rL \rfloor, m \leq \lfloor rd_{\text{in}} \rfloor, \text{ and } n \leq \lfloor rd_{\text{out}} \rfloor$. This step yields a compact set of coefficients, which we define as the \textbf{learngenes $\mathcal{G}$}. By discarding high-frequency information, this approach significantly reduces computing requirements compared to high-effort transfer while preserving essential knowledge. This makes direct extraction a scalable and lightweight solution for initializing downstream models.

% \subsection{FRONT{\scalebox{1}{\text{+}}}: Refinement with Frequency Regularization}

\subsection{\texorpdfstring{FRONT{\scalebox{1}{\text{+}}}}{FRONT+}: Extraction via Refinement}
\label{sec:regularization}
% While direct extraction is efficient, its reliance on hard truncation presents two key limitations. First, the abrupt cutoff can introduce spectral artifacts, disrupting the smoothness of the learned representations. Second, by completely discarding high-frequency coefficients, this approach risks eliminating fine-grained details that may contribute to the overall coherence of the knowledge. 
Direct extraction establishes a robust and scalable baseline. However, to further elevate performance, we address the inherent trade-offs associated with hard truncation. Specifically, the abrupt spectral cutoff can induce reconstruction artifacts, while the indiscriminate exclusion of high-frequency components risks discarding fine-grained details essential for capturing subtle feature variations.

To overcome these limitations, we introduce a refinement strategy that employs a smooth frequency regularization, which can be implemented from scratch or, fine-tuning a pre-trained model within minimal computational overhead. Before the hard cutoff, this method progressively suppresses high-frequency components while allowing knowledge to converge towards lower frequencies. Specifically, we penalize the energy of high-frequency coefficients during the model's optimization process:
% \begin{equation}
%     \label{reg}
%     {
%     \mathcal{L}_{\text{reg}} = \sum_{{\Phi}^i \in {\Omega}} \left(
%         \frac{1}{\|{M}_{r\text{+}} \odot {\Phi}^i\|_0}
%         \sum_{l,m,n} ({M}_{r\text{+}}[l,m,n] \cdot {\Phi}^i[l,m,n])^2
%     \right),}
% \end{equation}
\begin{equation}
    \label{reg}
    \resizebox{0.96\hsize}{!}{$ % 0.9\hsize 表示缩放到栏宽的 90%
    \mathcal{L}_{\text{reg}} = \sum_{{\Phi}^i \in {\Omega}} \left(
        \frac{1}{\|{M}_{r\text{+}} \odot {\Phi}^i\|_0}
        \sum_{l,m,n} ({M}_{r\text{+}}[l,m,n] \cdot {\Phi}^i[l,m,n])^2
    \right)
    $}
\end{equation}
where $\mathbf{M}_{r\text{+}}$ is a soft, dimension-wise penalty mask that assigns higher weight to higher frequencies:

\begin{equation}
\label{m}
    {M}_{r\text{+}}[l,m,n] = \prod_{d\in\{l,m,n\}} \left[1 - \exp\left(-\frac{f_d}{\gamma_d}\right) \right],
\end{equation}
where $f_l=\frac{l}{L}$, $f_m = \frac{m}{d_{\text{in}}}$, $f_n = \frac{n}{d_{\text{out}}}$  represent normalized frequency indices, and $\gamma_d$ are decay-rate hyperparameters. This formulation maintains gradient updates through all frequencies during refinement training while progressively attenuating high-frequency components. 
The total loss combines a task-specific objective $\mathcal{L}_{\text{total}}$ with our frequency regularization:
\begin{align}
    \mathcal{L}{_\text{total}} &= (1-\lambda)\mathcal{L}{_\text{task}} + \lambda\mathcal{L}{_\text{reg}} \label{eq:total_loss},
\end{align}
where $\mathcal{L}{_\text{task}}$ is the primary task training objective, such as cross-entropy for image classification or language modeling loss for pre-training transformers. Certainly, this term can be incorporated with other capability techniques, such as distillation loss~\citep{hinton2015distilling}, as detailed in Appendix~\ref{app:vit_a1}. After the refinement process, we extract the learngenes $\mathcal{G}$ using the same DCT and preservation steps outlined in Eq.(~\ref{eq:truncation}~). By progressively attenuating high frequencies, this approach generates smoother and more robust learngenes, which serve as a more effective foundation for initializing downstream models with varying sizes.

\subsection{Initialization of Variable-sized Models}  
\label{sec:init}
A key advantage of our framework is its ability to initialize target models of arbitrary sizes using a single set of learngenes.  The extracted learngenes $\mathcal{G}$ can be obtained through direct extraction by FRONT, or through refinement with frequency regularization by FRONT{\scalebox{1}{\text{+}}}. For target models with layers $L_{\text{tgt}}$ and hidden dimensions $d_{\text{tgt}}$, we adapt the learngenes through a  resizing operation entirely in the frequency domain, which can be performed \emph{in mere milliseconds on a CPU}:
\begin{equation}
\Phi_{\text{tgt}}^i = \mathcal{P}_{\text{pad}} \left(\mathcal{G}^i\right) \ \text{or}\ \mathcal{P}_\text{trunc}\left(\mathcal{G}^i\right) \in \mathbb{R}^{L_{\text{tgt}}\times d_{\text{tgt}}^\text{in}\times d_{\text{tgt}}^\text{out}},
\end{equation}
where $\mathcal{P}_\text{pad} $ applies zero padding to the high-frequency regions of the spectrum to increase tensor dimensions, while $ \mathcal{P}_\text{trunc}$ discards the outermost high-frequency coefficients to reduce them. The final weights are then reconstructed via the 3D-IDCT as $W_{\text{tgt}}^i = \mathcal{D}^{-1}(\Phi_{\text{tgt}}^i) \in \mathbb{R}^{L_{\text{tgt}}\times d_{\text{tgt}}^\text{in}\times d_{\text{tgt}}^\text{out}}.$

We provide an efficient, flexible framework for knowledge transfer by extracting knowledge into learngenes—via zero-overhead direct extraction or frequency-regularized refinement. Once learngenes are extracted, they are fully decoupled from the original models and enables initialization for models of varying depths and widths. Our approach overcomes conventional transfer limitations, promoting task-agnostic knowledge adaptation across diverse architectures.

\section{Experiments}
\label{others}

\subsection{Experimental Setup}

We systematically evaluate FRONT's effectiveness across vision and language tasks. \textbf{Scalability, }where we initialize models of varying depths and widths, and measure initialization quality by long-run training performance; \textbf{Generalization,} assessed via cross-dataset and cross-task transfer. The primary evaluation metric is Top-1 accuracy. See Appendix~\ref{app:all_training_details} and~\ref{app:additional_results} for more details and results.

\textbf{Vision.} We use ImageNet-1K pre-trained DeiT-Ti/S/B (12-layer) as source models. For refinement (\textbf{FRONT+}), we train 8-layer auxiliary models on ImageNet-1K for 150 epochs. Target models include DeiT variants with varying depths (4--12 layers) and widths (384--1536 dims), as well as ResNet and ViT backbones for seven classification, six object detection, and four image segmentation datasets. 
% Our target architectures include variable-sized DeiT, ViT, and ResNet models. We evaluate performance on ImageNet-1K~\cite{deng2009imagenet} and a diverse suite of downstream tasks, including seven classification, six object detection, and four image segmentation datasets. 
We benchmark our method against two main categories. (1) Direct initialization methods operate on existing pre-trained models without extra training, including He-Init~\citep{Chen_Xie_He_2021}, Mimetic-Init~\citep{trockman2023mimetic}, Wt Select~\citep{xu2023initializing}, Heur-LG~\citep{wang2022learngene}, Cluster-LG~\citep{wang2024clusterlearngene}, LiGO~\citep{wang2023learning}, and \textbf{FRONT{}}. (2) High-effort retraining methods include GHN-3~\citep{knyazev2023can}, Share-Init~\citep{lan2019albert}, Auto-LG~\citep{wang2023learngene}, TELG~\citep{xia2024transformer}, and WAVE~\citep{feng2024wave}, which require a costly, specialized process. For a fair comparison, our main experiments also train from scratch for \textbf{FRONT{\scalebox{1}{\text{+}}}} against this category. We also analyze our highly efficient fine-tuning variant, \textbf{FRONT{\scalebox{0.8}{\text{++}}}}, detailed in Section~\ref{sec:refine}.

\textbf{Language.} We demonstrate FRONT's applicability from ``Base'' (12-layer) to ``Small'' (6-layer) models for BERT~\citep{devlin-etal-2019-bert}, RoBERTa~\citep{liu2019roberta}, and GPT2~\citep{radford2019language} trained on standard corpora~\citep{devlin-etal-2019-bert} and evaluated on GLUE~\citep{wang2018glue}. Baselines for language tasks include training from scratch and knowledge distillation (KD)~\citep{hinton2015distilling}.

\subsection{Results on Vision Tasks}
\subsubsection{Initialization on Variable-sized Models}
% \textbf{Initialization Ability on Variable-sized Models.}
\label{sec:init_ability}
\renewcommand{\arraystretch}{0.81}
\begin{table*}[ht]
    \centering
    \setlength{\tabcolsep}{1.5mm}
    % \vspace{-0.1in}
    \caption{Performance of models initialized with different depths on ImageNet-1K for 10 epochs.  Para.(M) denotes the number of parameters, specified per model size (rows) and as the average transferred during initialization (columns). ``Epoch'' indicates the \textbf{extra training epochs} for knowledge merging or pre-training.  `` / '' denotes failed initialization, and ``N/A'' indicates that the metric is not applicable. The \textbf{best} result in each column is highlighted in bold, while the \underline{second-best} is underlined. 
    % All models were trained for 10 epochs following initialization.
    }
    \vspace{-0.06in}
    \setlength{\tabcolsep}{3pt}
    \resizebox{\textwidth}{!}{
        \begin{tabular}{@{}llccccccc|cccccc|cccccc@{}}
        \toprule[1.2pt]
        & & & & \multicolumn{5}{c|}{$W_{\text{192}}$ (DeiT-Ti)} & & \multicolumn{5}{c|}{$W_{\text{384}}$ (DeiT-S)} & & \multicolumn{5}{c}{$W_{\text{768}}$ (DeiT-B)}\\
        \cmidrule{5-9}
        \cmidrule{11-15}
        \cmidrule{17-21}
        & & & & $L_4$ & $L_6$ & $L_8$ & $L_{10}$ & $L_{12}$ & & $L_4$ & $L_6$ & $L_8$ & $L_{10}$ & $L_{12}$ & & $L_4$ & $L_6$ & $L_8$ & $L_{10}$ & $L_{12}$ \\
        \cmidrule{5-9}
        \cmidrule{11-15}
        \cmidrule{17-21}
        \multicolumn{2}{l}{Methods} & Epoch & Para. & \cellcolor{gray!15}{2.2} & \cellcolor{gray!15}{3.1} & \cellcolor{gray!15}{4.0} & \cellcolor{gray!15}{4.9} & \cellcolor{gray!15}{5.8} 
            & Para. & \cellcolor{gray!15}{7.9} & \cellcolor{gray!15}{11.5} & \cellcolor{gray!15}{15.0} & \cellcolor{gray!15}{18.6} & \cellcolor{gray!15}{22.2} 
            & Para. & \cellcolor{gray!15}{29.9} & \cellcolor{gray!15}{44.1} & \cellcolor{gray!15}{58.3} & \cellcolor{gray!15}{72.5} & \cellcolor{gray!15}{86.7} \\
        % \midrule
        \toprule[1.2pt]
        \multirow{7}{*}{\rotatebox{90}{\small{\textbf{Direct}}}}
        & He-Init & 0
                & \cellcolor{gray!15}{0} & 34.7 & 40.6 & 43.7 & 46.8 & 48.3
                & \cellcolor{gray!15}{0} & 42.2 & 49.4 & 52.1 & 53.7 & 55.5
                & \cellcolor{gray!15}{0} & 47.9 & 53.1 & 54.4 & 55.0 & 56.7 \\
        & Mimetic & 0
                & \cellcolor{gray!15}{0} & 35.1 & 40.2 & 43.2 & 46.3 & 48.1
                & \cellcolor{gray!15}{0} & 43.3 & 49.1 & 53.0 & 54.1 & 55.6
                & \cellcolor{gray!15}{0} & 50.2 & 54.3 & 56.5 & 58.5 & 58.6 \\
        & Wt Select & 0
             & \cellcolor{gray!15}{3.9} & 46.1 & 52.3 & 55 & 56.8 & 58.9
             & \cellcolor{gray!15}{15.0} & 50.1 & 56 & 57.5 & 61.9 & 63.3
             & \cellcolor{gray!15}{58.2} & 56.7 & 61.1 & 63.6 & 64.8 & 65.9 \\
        & Heur-LG & 0
                & \cellcolor{gray!15}{1.7} & 41.5 & 47.4 & 50.5 & 53.5 & 55.5
                & \cellcolor{gray!15}{6.1} & 52.3 & 57.3 & 61.7 & 64.4 & 65.9
                & \cellcolor{gray!15}{22.8} & \underline{60.5} & 68.7 & 72.2 & 73.6 & 74.0 \\
        & Cluster-LG & 0
                & \cellcolor{gray!15}{2.4} & \underline{46.6} & 51.5 & 51.9 & 55.8 & 57.2
                & \cellcolor{gray!15}{5.1} & \underline{53.2} & 54.9 & 52.6 & 52.1 & 55.5
                & \cellcolor{gray!15}{19.6} & 52.7 & 60.4 & 60.6 & 58.6 & 62.7 \\
        & LiGO & 0
             & \cellcolor{gray!15}{2.2} & / & \underline{59.0} & \underline{60.2} & \underline{59.8} & \underline{60.9}
             & \cellcolor{gray!15}{7.9} & / & \underline{68.6} & \underline{69.9} & \underline{69.7} & \underline{70.0}
             & \cellcolor{gray!15}{29.9} & / & \underline{74.2} & \underline{74.4} & \underline{75.3} & \underline{75.4} \\
        % \rowcolor{gray!15}
        & \cellcolor{blue!12}{FRONT{}} & \cellcolor{blue!12}{0}
               & \cellcolor{blue!12}{2.2} & 
               \cellcolor{blue!12}{\textbf{55.3}} & \cellcolor{blue!12}{\textbf{63.3}} & \cellcolor{blue!12}{\textbf{64.4}} & \cellcolor{blue!12}{\textbf{64.7}} & \cellcolor{blue!12}{\textbf{65.3}}
               & \cellcolor{blue!12}{8.1} & \cellcolor{blue!12}{\textbf{63.4}} & \cellcolor{blue!12}{\textbf{68.9}} & \cellcolor{blue!12}{\textbf{70.5}} & \cellcolor{blue!12}{\textbf{70.9}} & \cellcolor{blue!12}{\textbf{71.2}}
               & \cellcolor{blue!12}{32.4} & \cellcolor{blue!12}{\textbf{72.1}} & \cellcolor{blue!12}{\textbf{74.6}} & \cellcolor{blue!12}{\textbf{75.3}} & \cellcolor{blue!12}{\textbf{75.8}} & \cellcolor{blue!12}{\textbf{76.3}}\\
        \midrule
        \multirow{6}{*}{\rotatebox{90}{\small{\textbf{Train}}}}
        & GHN-3 & N/A
                & \cellcolor{gray!15}{N/A} & 40.9 & 45.0 & 46.6 & 49.1 & 48.9
                & \cellcolor{gray!15}{N/A} & 45.4 & 49.0 & 50.2 & 52.3 & 53.2
                & \cellcolor{gray!15}{N/A} & 49.5 & 52.5 & 53.8 & 54.2 & 54.3 \\
        & Share Init & 150
                   & \cellcolor{gray!15}{0.8} & 55.2 & 59.8 & 62.5 & 64.3 & 65.3
                   & \cellcolor{gray!15}{2.5} & 65.0 & 69.7 & 71.7 & 72.7 & 73.3
                   & \cellcolor{gray!15}{8.6} & 71.7 & 75.3 & 76.4 & 77.4 & 77.6 \\

        & Auto-LG & 50
                & \cellcolor{gray!15}{2.2} & 52.4 & 61.8 & 64.6 & 65.9 & 66.8
                & \cellcolor{gray!15}{7.9} & 63.2 & 70.5 & 72.2 & 73.3 & 73.8
                & \cellcolor{gray!15}{29.9} & 60.9 & 70.0 & 72.4 & 73.5 & 73.8 \\
        & TLEG & 150
             & \cellcolor{gray!15}{1.3} & 55.0 & 60.5 & 62.9 & 64.4 & 65.4
             & \cellcolor{gray!15}{4.3} & 65.4 & 70.5 & 72.1 & 73.2 & 73.8
             & \cellcolor{gray!15}{15.7} & 71.6 & 74.9 & 76.2 & 77.0 & 77.1 \\
        % & WAVE~\citep{feng2024wave} & 150
        %      & \cellcolor{gray!15}{1.3} & 58.6 & 63.2 & 65.4 & 66.6 & 67.3
        %      & \cellcolor{gray!15}{4.3} & 68.9 & 72.7 & 74.1 & 74.9 & 75.3
        %      & \cellcolor{gray!15}{15.7} & 74.5 & 77.5 & 78.2 & 78.9 & 79.2 \\
        & WAVE & 150
                 & \cellcolor{gray!15}{1.3} & \underline{58.6} & \underline{63.2} & \underline{65.4} & \underline{66.6} & \underline{67.3}
                 & \cellcolor{gray!15}{4.3} & \textbf{68.9} & \underline{72.7} & \underline{74.1} & \underline{74.9} & \underline{75.3}
                 & \cellcolor{gray!15}{15.7} & \textbf{74.5} & \underline{77.5} & \underline{78.2} & \underline{78.9} & \underline{79.2} \\
        % \rowcolor{gray!15}
        & \cellcolor{blue!12}{FRONT{\scalebox{1}{\text{+}}}} & \cellcolor{blue!12}{150}
               & \cellcolor{blue!12}{0.8} & \cellcolor{blue!12}{\textbf{58.7}} & \cellcolor{blue!12}{\textbf{63.4}} & \cellcolor{blue!12}{\textbf{65.6}} & \cellcolor{blue!12}{\textbf{66.8}} & \cellcolor{blue!12}{\textbf{67.5}}
               & \cellcolor{blue!12}{3.2} & \cellcolor{blue!12}{\textbf{68.9}} & \cellcolor{blue!12}{\textbf{72.9}} & \cellcolor{blue!12}{\textbf{74.2}} & \cellcolor{blue!12}{\textbf{75.0}} & \cellcolor{blue!12}{\textbf{75.4}}
               & \cellcolor{blue!12}{13.0} & \cellcolor{blue!12}{\textbf{74.5}} & \cellcolor{blue!12}{\textbf{77.7}} & \cellcolor{blue!12}{\textbf{78.5}} & \cellcolor{blue!12}{\textbf{79.0}} & \cellcolor{blue!12}{\textbf{79.3}}\\
        \midrule
        \multirow{1}{*}{\rotatebox{90}{\small{\textbf{PT}}}}
        & Direct PT & \small{150$\times$15}
                & \cellcolor{gray!15}{4.0} & 50.4 & 57.7 & 62.7 & 66.2 & 68.6
                & \cellcolor{gray!15}{15.0} & 62.6 & 70.1 & 73.8 & 76.0 & 77.6
                & \cellcolor{gray!15}{58.3} & 70.7 & 76.2 & 79.1 & 80.5 & 81.5 \\
               
        \bottomrule[1.2pt]
        
        \end{tabular}
        }
    \label{tab:lenth}
    \vspace{-0.2in}
\end{table*}

\begin{table}[ht]  % r表示右侧，0.7\textwidth是表格宽度
    \centering
    \setlength{\tabcolsep}{1.5 mm}
    % \vspace{-0.17in}
    \caption{Performance of initializing models with variable widths on ImageNet-1K. 
    % for 10 epochs.
    }
    \vspace{-0.1in}
    \resizebox{0.48\textwidth}{!}{  % 调整为适合wraptable的宽度
        \begin{tabular}{@{}llccccccc@{}}
        % 表格内容保持不变
        \toprule[1.2pt]
        & & & & \multicolumn{5}{c}{$L_{\text{6}}$}\\
        \cmidrule{5-9}
        & & & & $W_{\text{384}}$ & $W_{\text{576}}$ & $W_{\text{768}}$ & $W_{\text{1152}}$ & $W_{\text{1536}}$ \\
        \cmidrule{5-9}
        \multicolumn{2}{l}{Methods} & Epoch & Para. & \cellcolor{gray!15}{11.5} & \cellcolor{gray!15}{25.1} & \cellcolor{gray!15}{44.1} & \cellcolor{gray!15}{98.0} & \cellcolor{gray!15}{173.1} \\ 
        \toprule[1.0pt]
        \multirow{5}{*}{\rotatebox{90}{\small{\textbf{Direct}}}}
        & He-Init & 0 & \cellcolor{gray!15}{0} & 49.4 & 51.5 & 53.1 & 46.6 & 31.3 \\
        & Mimetic & 0 & \cellcolor{gray!15}{0} & 49.1 & 48.0 & 54.3 & 47.7 & 33.0 \\ 
         & Wt Select & 0 & \cellcolor{gray!15}{26.8} &  48.6 & 50.7 & 55.4 & / & /  \\
        & LiGO & 0 & \cellcolor{gray!15}{10.0} & \textbf{63.6} & \underline{63.1} & \underline{69.5} & \underline{70.0} & \underline{73.7} \\
        & \cellcolor{blue!12}{FRONT{}} & \cellcolor{blue!12}{0} & \cellcolor{blue!12}{8.9} & \cellcolor{blue!12}{\underline{62.3}} & \cellcolor{blue!12}{\textbf{63.6}} & \cellcolor{blue!12}{\textbf{69.8}} & \cellcolor{blue!12}{\textbf{70.7}} & \cellcolor{blue!12}{\textbf{73.9}}  \\
        \midrule
        \multirow{3}{*}{\rotatebox{90}{\small{\textbf{Train}}}}
        & GHN-3 & N/A & \cellcolor{gray!15}{N/A} & 49.0 & 51.8 & 52.5 & 52.7 & 51.0 \\
        & WAVE & 150 & \cellcolor{gray!15}{5.4} & \underline{64.8} & \underline{67.0} & \underline{72.4} & \underline{73.5} & \underline{78.3} \\
        & \cellcolor{blue!12}{FRONT{\scalebox{1}{\text{+}}}} & \cellcolor{blue!12}{150} & \cellcolor{blue!12}{5.5} & \cellcolor{blue!12}{\textbf{66.3}} & \cellcolor{blue!12}{\textbf{68.3}} & \cellcolor{blue!12}{\textbf{73.5}} & \cellcolor{blue!12}{\textbf{74.6}} & \cellcolor{blue!12}{\textbf{78.7}}  \\
        \bottomrule[1.2pt]
        \end{tabular}
    }
    \label{tab:width}
    \vspace{-0.17in}
\end{table}

\textbf{Depth Expansion.}
We systematically evaluate the initialization of variable-sized models in Table~\ref{tab:lenth}, a crucial requirement for resource-constrained deployment, by benchmarking performance on ImageNet-1K. Our method consistently outperforms both direct processing and learngene approaches. Methods leveraging pre-trained models surpass training from scratch, highlighting the benefits of leveraging prior knowledge. While LiGO improves over Mimetic, its injection of random parameters leads to model size mismatch. In contrast, FRONT{} exploits low-frequency knowledge from pre-trained models, surpassing LiGO by 2\% and He-Init by 19.8\%.

Compared to methods requiring additional training, Share Init outperforms GHN3 by combining rule-based priors with pre-trained layers. 
In contrast, several learngene methods like Auto-LG and TLEG impose rigid layer-specific constraints, limiting scalability across models.
While WAVE performs competitively via weight templates, FRONT{\scalebox{1}{\text{+}}} surpasses it while transferring over 25\% fewer parameters. Notably, FRONT{\scalebox{1}{\text{+}}} outperforms 150-epoch pretraining after only 10 epochs, highlighting its efficiency and flexibility. We also demonstrated our training curve in Figure~\ref{fig:curve}.

\textbf{Width Expansion.} The reversible nature of DCT and IDCT enables FRONT to flexibly scale model widths. As shown in Table~\ref{tab:width}, our method consistently outperforms other initialization and transfer methods. Direct initialization methods like Mimetic exhibit substantially lower performance, underscoring the critical role of pre-trained knowledge in effective model initialization.
LiGO introduces excessive randomness, hindering the initialization of larger models, while Wt Select disrupts structural knowledge transfer. Although WAVE utilizes Kronecker operations with weight templates for reasonable results, FRONT{\scalebox{1}{\text{+}}}’s transformation of low-frequency information further enhances source models' decoupling. This allows our method to achieve robust and flexible initialization across varying model widths while preserving essential knowledge structures.

\begin{table*}
    \centering 
    \setlength{\tabcolsep}{1.5mm}
    % \vspace{-0.1in}
    \caption{Performance of models on downstream datasets. ``Para.(M)'' is the average parameter transferred during model initialization.}
    \vspace{-0.1in}
    \resizebox{\textwidth}{!}{
        \begin{tabular}{@{}llcccccccc>{\columncolor{gray!15}}c|cccccccc>{\columncolor{gray!15}}c@{}}
        \toprule[1.2pt]
             & & \multicolumn{9}{c|}{DeiT-Ti, 3.0M} & \multicolumn{9}{c}{DeiT-S, 11.3M}\\
             \cmidrule{3-20}
             \multicolumn{2}{l}{Methods} & Para. & Flow. & CUB & Cars & C\small{10} & C\small{100} & Food & iNat. & \textit{Aver.}
             & Para. & Flow. & CUB & Cars & C\small{10} & C\small{100} & Food & iNat. & \textit{Aver.}\\
             % \hline
             \midrule[1.1pt]
             \multirow{7}{*}{\rotatebox{90}{\small{\textbf{Direct}}}}
             & He-Init & 0 & 53.9 & 26.1 & 19.9 & 92.4 & 68.3 & 68.4 & 52.3 & \textit{54.5}
                     & 0 & 57.2 & 27.3 & 23.8 & 94.0 & 66.5 & 70.6 & 54.0 & \textit{56.2}\\
             & Mimetic & 0 & 52.1 & 35.0 & 20.5 & 88.9 & 63.4 & 66.9 & 49.0 & \textit{53.7}
                          & 0 & 57.4 & 39.6 & 34.2 & 91.6 & 65.7 & 67.1 & 52.2 & \textit{58.3}\\ 
             
             & Wt Select & 2.9 & 55.0 & 34.4 & 21.7 & 92.5 & 67.0 & 67.4 & 51.2 & \textit{55.6}
                           & 11.0 & 58.3 & 33.1 & 28.1 & 94.1 & 68.1 & 69.0 & 54.2 & \textit{57.8}\\
            & Heur-LG & 1.5 & 64.7 & 44.6 & 37.7 & 94.0 & 71.1 & 74.7 & 57.4 & \textit{63.5}
                     & 5.7 & 69.1 & 48.0 & 51.2 & 95.1 & 72.8 & 76.8 & 59.3 & \textit{67.5}\\
            & Cluster-LG & 2.4 & 62.8 & 44.8 & 36.9 & 94.1 & 74.5 & 78.6 & 56.4 & \textit{64.0}
                      & 5.1 & 65.5 & 45.4 & 46.3 & 95.4 & 74.6 & 78.4 & 60.4 & \textit{66.6}\\
            
            & LiGO & 2.0 & \textbf{94.2} & \textbf{71.8} & \textbf{83.9} & \underline{95.6} & \underline{78.5} & \underline{82.1} & \underline{61.6} & \textbf{\textit{81.1}}
                        & 7.5 & \textbf{95.9} & \textbf{74.8} & \textbf{87.9} & \underline{96.9} & \underline{81.3} & \underline{84.0} & \underline{66.1} & \underline{\textit{83.8}}\\
            %实验正在更新，还未填进去    
            & \cellcolor{blue!12}{FRONT{}} & \cellcolor{blue!12}{2.2} & \cellcolor{blue!12}{\underline{92.9}} & \cellcolor{blue!12}{\underline{70.5}} & \cellcolor{blue!12}{\underline{82.3}} & \cellcolor{blue!12}{\textbf{95.9}} & \cellcolor{blue!12}{\textbf{78.7}} & \cellcolor{blue!12}{\textbf{83.7}} & \cellcolor{blue!12}{\textbf{63.7}} & \cellcolor{blue!12}{\textbf{\textit{81.1}}}
                  & \cellcolor{blue!12}{8.1} & \cellcolor{blue!12}{\underline{94.5}} & \cellcolor{blue!12}{\underline{74.1}} & \cellcolor{blue!12}{\underline{87.4}} & \cellcolor{blue!12}{\textbf{97.2}} & \cellcolor{blue!12}{\textbf{81.8}} & \cellcolor{blue!12}{\textbf{85.2}} & \cellcolor{blue!12}{\textbf{66.6}} & \cellcolor{blue!12}{\textbf{\textit{83.9}}}\\
             \midrule
             \multirow{6}{*}{\rotatebox{90}{\small{\textbf{Train}}}}
             & GHN-3 & N/A & 50.0 & 41.1 & 23.2 & 92.5 & 70.1 & 76.2 & 51.7 & \textit{57.8}
                   & N/A & 52.7 & 45.2 & 30.6 & 93.9 & 72.7 & 76.2 & 55.5 & \textit{61.0}\\
             & Share Init & 0.6 & 92.4 & 70.1 & 82.1 & 96.0 & 77.2 & 81.2 & 63.2 & \textit{80.3}
                        & 2.2 & 94.1 & 72.4 & 87.2 & 96.5 & 78.5 & 83.0 & 63.0 & \textit{82.1}\\
             
             & Auto-LG & 2.0 & 93.5 & 71.4 & 83.5 & 96.4 & 77.1 & 81.7 & 62.5 & \textit{80.9}
                     & 7.5 & 96.4 & 75.1 & 88.2 & 97.3 & 81.0 & 84.6 & 67.0 & \textit{84.2}\\
             & TLEG & 1.1 & 91.0 & 69.5 & 78.2 & 96.1 & 77.0 & 82.0 & 63.4 & \textit{79.6}
                  & 3.9 & 93.7 & 72.6 & 87.2 & 97.2 & 80.2 & 84.9 & 66.5 & \textit{83.2} \\
            & WAVE &1.1& \underline{94.9} & \underline{74.8} & \underline{84.4} & \underline{96.6} & \underline{80.7} & \underline{83.8} & \underline{65.2} & \textit{\underline{82.9}}
                  & 4.0 & \underline{96.9} & \underline{78.1} & \underline{89.4} & \underline{97.4} & \underline{83.2} & \underline{85.5} & \underline{67.6} & \textit{\underline{85.4}} \\

             & \cellcolor{blue!12}{FRONT{\scalebox{1}{\text{+}}}} & \cellcolor{blue!12}{0.8} & \cellcolor{blue!12}{\textbf{95.1}} & \cellcolor{blue!12}{\textbf{75.2}} & \cellcolor{blue!12}{\textbf{86.1}} & \cellcolor{blue!12}{\textbf{96.6}} & \cellcolor{blue!12}{\textbf{80.8}} & \cellcolor{blue!12}{\textbf{83.9}} & \cellcolor{blue!12}{\textbf{65.4}} & \cellcolor{blue!12}{\textbf{\textit{83.3}}}
                  & \cellcolor{blue!12}{3.2} & \cellcolor{blue!12}{\textbf{97.2}} & \cellcolor{blue!12}{\textbf{78.2}} & \cellcolor{blue!12}{\textbf{89.4}} & \cellcolor{blue!12}{\textbf{97.4}} & \cellcolor{blue!12}{\textbf{84.0}} & \cellcolor{blue!12}{\textbf{86.1}} & \cellcolor{blue!12}{\textbf{68.1}} & \cellcolor{blue!12}{\textbf{\textit{85.8}}}\\
             \midrule
             
             \multirow{1}{*}{\rotatebox{90}{\small{\textbf{PT}}}} 
             & Direct FT & 2.9 & 95.4 & 75.1 & 86.5 & 96.6 & 80.2 & 84.0 & 66.9 & \textit{83.5}
                      & 11.0 & 96.4 & 77.0 & 89.4 & 97.5 & 82.8 & 85.6 & 69.3 & \textit{85.4}\\

             \bottomrule[1.2pt]
        \end{tabular}
    \label{tab:downstream}
    }
    % \vspace{-0.1in}
\end{table*}

% \begin{table*}
%     \centering 
%     \setlength{\tabcolsep}{0.8mm}
%     % \vspace{-0.1in}
%     \caption{Object Detection (ViT-10shot) SD: COCO → TD: Below}
%     \vspace{-0.1in}
%     \resizebox{\textwidth}{!}{
%         \begin{tabular}{@{}llcccccc>{\columncolor{gray!15}}c|cccccc>{\columncolor{gray!15}}c@{}}
%         \toprule[1.2pt]
%              & & \multicolumn{7}{c|}{ViT-S} & \multicolumn{7}{c}{ViT-B}\\
%              \multicolumn{2}{l}{Methods} & ArTaxOr & Clipart1k & DIOR & DeepFish & NEU-DET & UODD & \textit{Aver.}
%               & ArTaxOr & Clipart1k & DIOR & DeepFish & NEU-DET & UODD& \textit{Aver.}\\
%              \midrule[1.1pt]
%              & Scratch& 41.97 & 27.70 & 21.40 & 20.00 & 12.57 & \textbf{5.15} & \textit{21.47}
%                      & 51.42 & 37.59 & 24.50 & 19.22 & 12.99 & \textbf{5.81} & \textit{25.26}\\
%              & \cellcolor{blue!12}{FRONT{}} & \cellcolor{blue!12}{\textbf{42.36}} & \cellcolor{blue!12}{\textbf{29.74}} & \cellcolor{blue!12}{\textbf{23.33}} & \cellcolor{blue!12}{\textbf{20.73}} & \cellcolor{blue!12}{\textbf{12.71}} & \cellcolor{blue!12}{5.09} & \cellcolor{blue!12}{\textit{\textbf{22.33}}}
%                        & \cellcolor{blue!12}{\textbf{53.48}} & \cellcolor{blue!12}\textbf{{39.49}} & \cellcolor{blue!12}\textbf{{25.48}} & \cellcolor{blue!12}\textbf{{19.71}} & \cellcolor{blue!12}\textbf{13.21} & \cellcolor{blue!12}{5.78} &  \cellcolor{blue!12}{\textit{\textbf{26.19}}}\\ 
%              \bottomrule[1.2pt]
%         \end{tabular}
%     \label{tab:object_detection}
%     }
%     \vspace{-0.15in}
% \end{table*}

\begin{figure}[t] % [t] 表示尽量放在栏目顶部，也可以用 [b] 底部或 [h] 当前位置
    \centering
    % 第一个子图
    \begin{subfigure}{0.49\linewidth} % 宽度设为行宽的约一半，以便并排
        \centering
        \includegraphics[width=\linewidth]{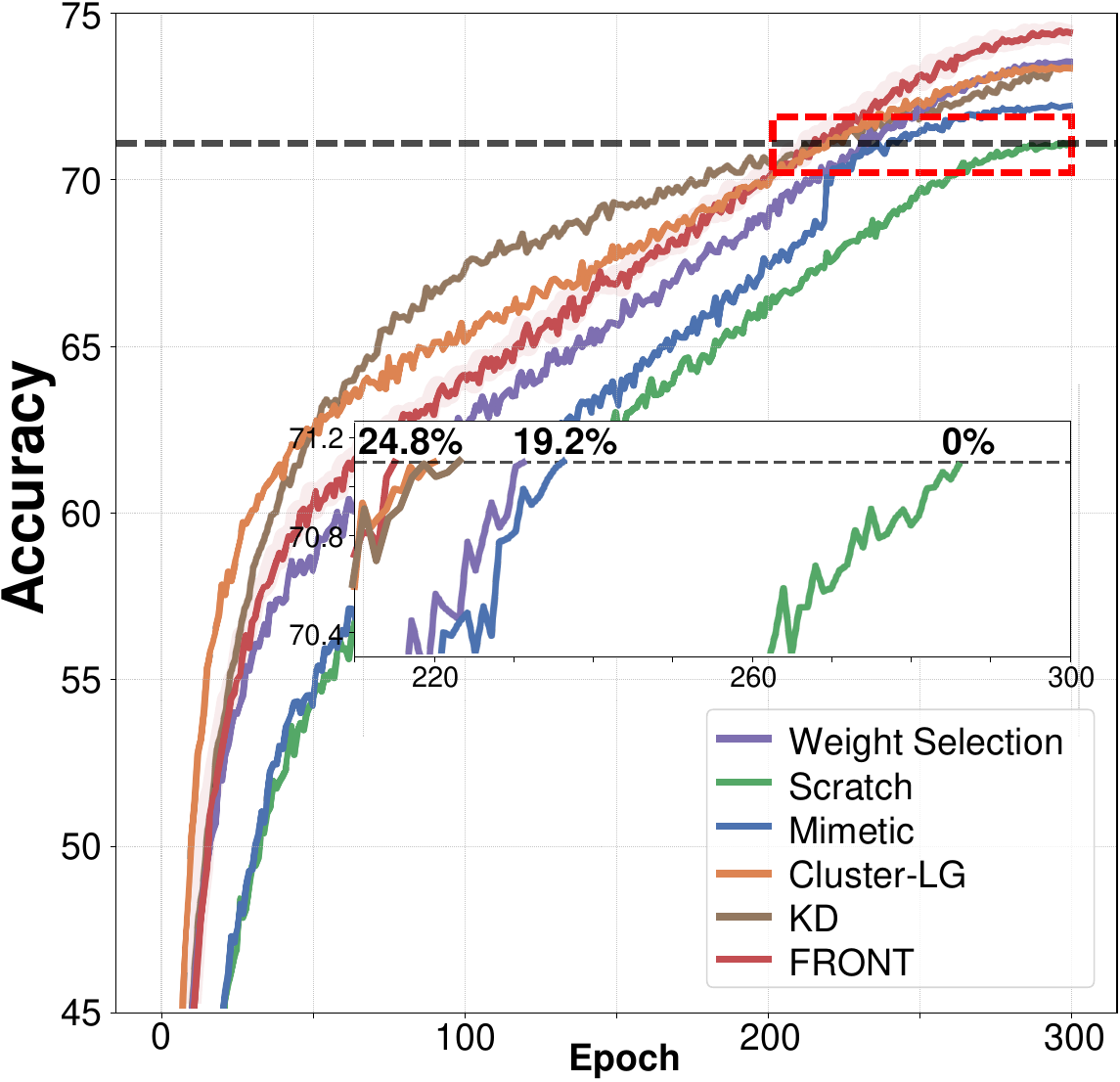}
        \caption{DeiT-Ti}
        \label{fig:init_acc_trend}
    \end{subfigure}%
    \hfill % 在两个子图之间添加弹簧空格，把它们撑到两边
    % 第二个子图
    \begin{subfigure}{0.49\linewidth}
        \centering
        \includegraphics[width=\linewidth]{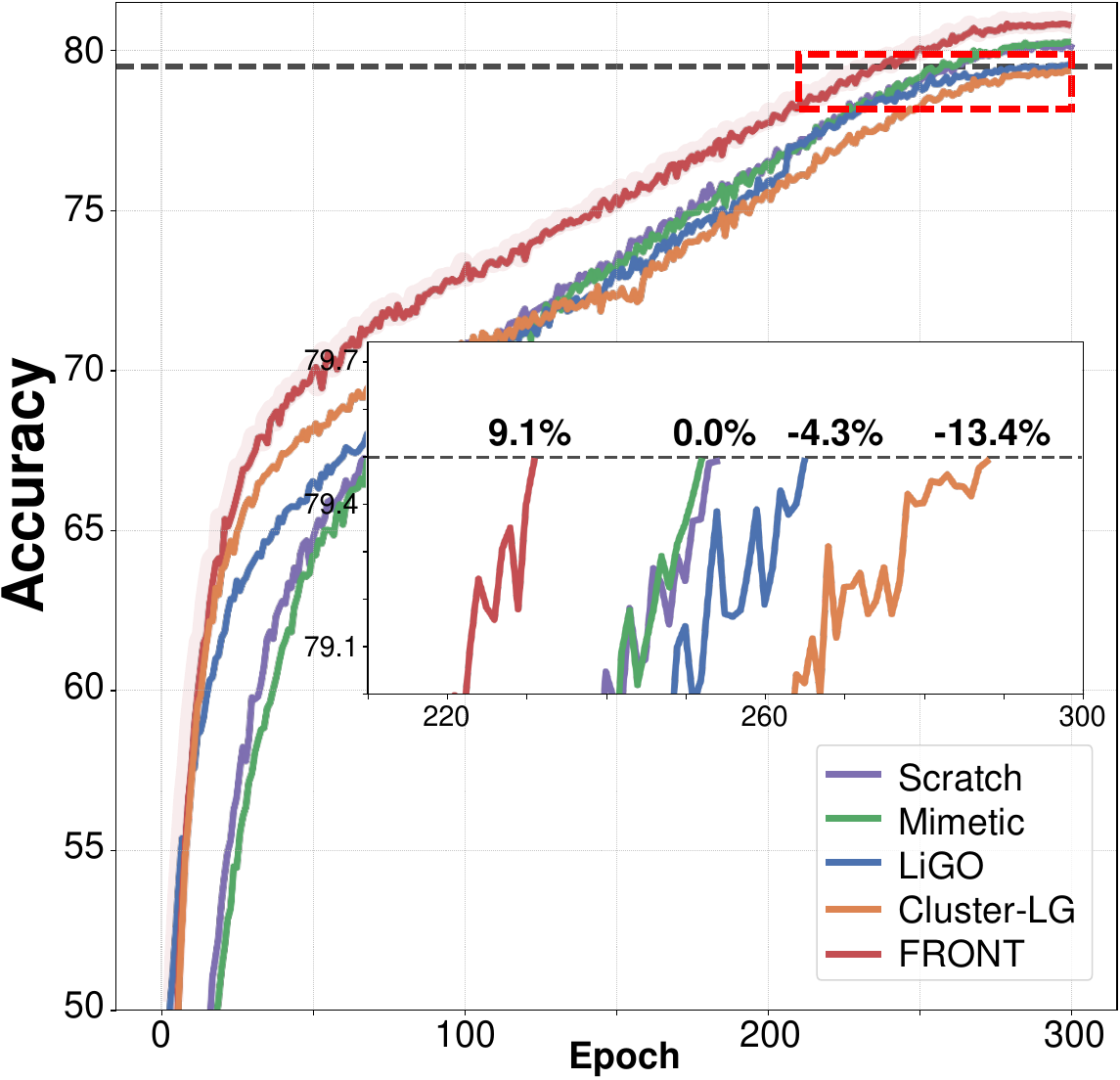}
        \caption{DeiT-S}
        \label{fig:weight_sel_perf}
    \end{subfigure}
    
    \vspace{-0.1in} % 调整图片与caption的距离（可选）
    \caption{Comparison of different direct initialization methods during full training on DeiT-Ti and DeiT-S.}
    \label{fig:DeiT_300} % 建议修改label名，不再是wrap了
    \vspace{-0.27in} % 调整caption与下方正文的距离（可选）
\end{figure}

\textbf{Initialization Performance during Full Training.} Figure~\ref{fig:DeiT_300} shows the training performance of various direct initialization methods between DeiT-Ti and DeiT-S over 300 epochs. All reported results for FRONT are averaged over five independent runs with different random seeds to ensure robustness. Methods leveraging pre-trained knowledge, including Wt Select, LiGO, KD and Cluster-LG, exhibit stronger early-stage performance than random initialization.
In contrast, FRONT{} employs a unified conversion strategy that preserves low-frequency components, enabling effective initialization and consistently strong performance throughout training. 

% It achieves the highest final accuracy with minimal parameter transfer without extra cost, and its consistent results further highlight its robustness.

% \begin{figure}[ht]
% \centering

% \includegraphics[width=0.48\textwidth]{img/DeiT_300.pdf}
% \vspace{-0.12in}
% \caption{Accuracy curves with different methods over 300 epochs.} % 图片标题
% \vspace{-0.25in}
% \label{fig:DeiT_300}
% \end{figure}

\begin{table}[ht]
  \centering % 让表格和标题在分配的宽度内居中
  \setlength{\tabcolsep}{4.8mm}

  \caption{GLUE benchmark results on BERT-S.}
  \label{tab:glue_results1}
  \vspace{-0.1in}
  \resizebox{0.96\linewidth}{!}{
  \begin{tabular}{llll}
    \toprule
    Task & Scratch & \begin{tabular}[c]{@{}c@{}}KD\end{tabular} & {FRONT{}}  \\
    \midrule
    SST-2   & 78.97 & 77.75 & \textbf{82.44} \textcolor{mygreen}{($\uparrow$ 3.47)}\\
    MNLI    & 63.65 & 64.81 & \textbf{74.08} \textcolor{mygreen}{($\uparrow$ 9.27)}\\
    MRPC    & 66.42 & 64.14 & \textbf{68.46} \textcolor{mygreen}{($\uparrow$ 2.04)}\\
    CoLA    & 8.14  & 9.14  & \textbf{17.38} \textcolor{mygreen}{($\uparrow$ 8.24)}\\
    QNLI    & 57.21 & 58.00 & \textbf{70.86} \textcolor{mygreen}{($\uparrow$ 13.65)}\\
    QQP     & 81.50 & 80.07 & \textbf{84.93} \textcolor{mygreen}{($\uparrow$ 4.86)}\\
    STS-B   & 17.27 & 15.76 & \textbf{54.23} \textcolor{mygreen}{($\uparrow$ 36.96)}\\
    \midrule
    Average & 53.31 & 52.81 & \textbf{64.63} \textcolor{mygreen}{($\uparrow$ 11.32)}\\
    \bottomrule
  \end{tabular}%
  } % resizebox 的结尾
  \vspace{-0.18in} % 调整环绕区域底部的间距
\end{table}

\subsubsection{Generalization Across Downstream Datasets}

\textbf{Initialization on DeiT-based Architectures.} In Table~\ref{tab:downstream}, all methods, including ours, use a source model pre-trained on ImageNet-1K for initialization. Following that, the downstream models are directly trained on downstream tasks without any additional pre-training.

Our method consistently improves performance across all datasets, demonstrating strong generalization from initialization. In contrast, methods such as Mimetic and GHN-3 show limited adaptability on certain datasets, while Wt Select occasionally underperforms He-Init on specific datasets like iNaturalist-2019. These results emphasize the need to transfer essential knowledge without task overfitting. Notably, FRONT{\scalebox{1}{\text{+}}} achieves superior performance while transferring fewer parameters than WAVE. By efficiently leveraging low-frequency components, our method reduces storage overhead and enables flexible, effective initialization.

\textbf{Initialization on ResNet-based Architectures.} 
To further assess the generality of FRONT, we conduct experiments on CNNs using ResNet models. A community-trained ResNet-101 on ImageNet-1K is used to initialize both ResNet-50 and ResNet-152, which are then evaluated on seven downstream tasks directly and compared against randomly initialized baselines. Results are shown in Table~\ref{tab:performance_comparison}.

\begin{table}[ht]
  \centering % 让表格和标题在分配的宽度内居中
  \setlength{\tabcolsep}{2.8mm}

\vspace{-0.17in}
\caption{
% Performance comparison between random initialization and FRONT on various dataset.
Performance comparison on ResNet50 and ResNet152.
}
% \vspace{-0.12in}
\resizebox{\linewidth}{!}{ % 这里让表格宽自动等于wraptable的宽
\begin{tabular}{lccl|ccl}
\toprule
    & \multicolumn{3}{c|}{ResNet50} & \multicolumn{3}{c}{ResNet152} \\
    \cmidrule{2-7}
    & Rand. & \multicolumn{2}{c|}{FRONT{}} & Rand. & \multicolumn{2}{c}{FRONT{}} \\
    \midrule
    Image. & 74.1 & 76.1 & \textcolor{mygreen}{$\uparrow$ 2.0} & 76.8 & 77.2 & \textcolor{mygreen}{$\uparrow$ 0.4} \\
    Flow. & 51.3 & 91.3 & \textcolor{mygreen}{$\uparrow$ 40.0} & 28.0 & 87.7 & \textcolor{mygreen}{$\uparrow$ 59.7} \\
    CUB & 40.7 & 68.9 & \textcolor{mygreen}{$\uparrow$ 28.2} & 22.8 & 68.2 & \textcolor{mygreen}{$\uparrow$ 45.4} \\
    Cars & 48.6 & 88.2 & \textcolor{mygreen}{$\uparrow$ 39.6} & 29.3 & 87.4 & \textcolor{mygreen}{$\uparrow$ 58.1} \\
    C10 & 94.8 & 95.6 & \textcolor{mygreen}{$\uparrow$ 0.8} & 95.6 & 95.7 & \textcolor{mygreen}{$\uparrow$ 0.1} \\
    C100 & 76.8 & 78.3 & \textcolor{mygreen}{$\uparrow$ 1.5} & 77.7 & 79.5 & \textcolor{mygreen}{$\uparrow$ 1.8} \\
    Food & 82.1 & 85.8 & \textcolor{mygreen}{$\uparrow$ 3.7} & 83.2 & 85.9 & \textcolor{mygreen}{$\uparrow$ 2.7} \\
    iNat & 61.8 & 69.0 & \textcolor{mygreen}{$\uparrow$ 7.2} & 62.7 & 69.1 & \textcolor{mygreen}{$\uparrow$ 6.4} \\
\bottomrule
\end{tabular}
}
\label{tab:performance_comparison}
\vspace{-0.12in}
\end{table}

FRONT consistently outperforms random initialization across all datasets, with particularly notable improvements on fine-grained, small-sample datasets. For example, on the Flowers dataset, FRONT boosts accuracy by 40.0\% (ResNet-50) and 59.7\% (ResNet-152); on CUB, by 28.2\% and 45.4\%; and on Stanford Cars, by 39.6\% and 58.1\%, respectively. These results highlight FRONT’s ability to address the difficulty of training deep networks on limited data by leveraging pre-trained models. FRONT thus provides more effective initialization and demonstrates strong generalization, enabling substantial gains even on datasets with unique characteristics and CNN-based architectures.

\begin{figure*}[htbp]

    \centering
    \subfloat[BERT-B $ \rightarrow$ BERT-S\label{fig:a1}]{%
        \includegraphics[width=0.333\textwidth]{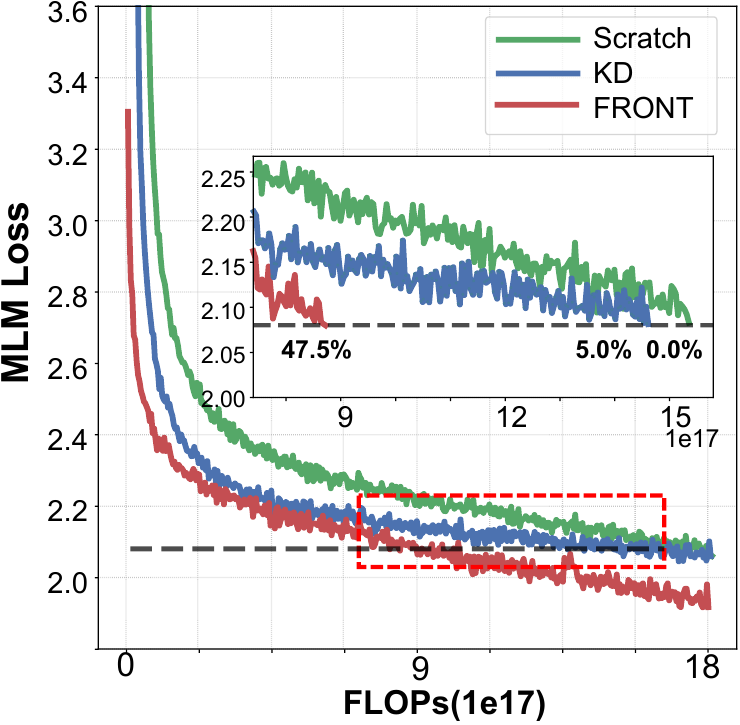}%
    }
    \hfill
    \subfloat[RoBERTa-B $ \rightarrow$ RoBERTa-S\label{fig:b1}]{%
        \includegraphics[width=0.333\textwidth]{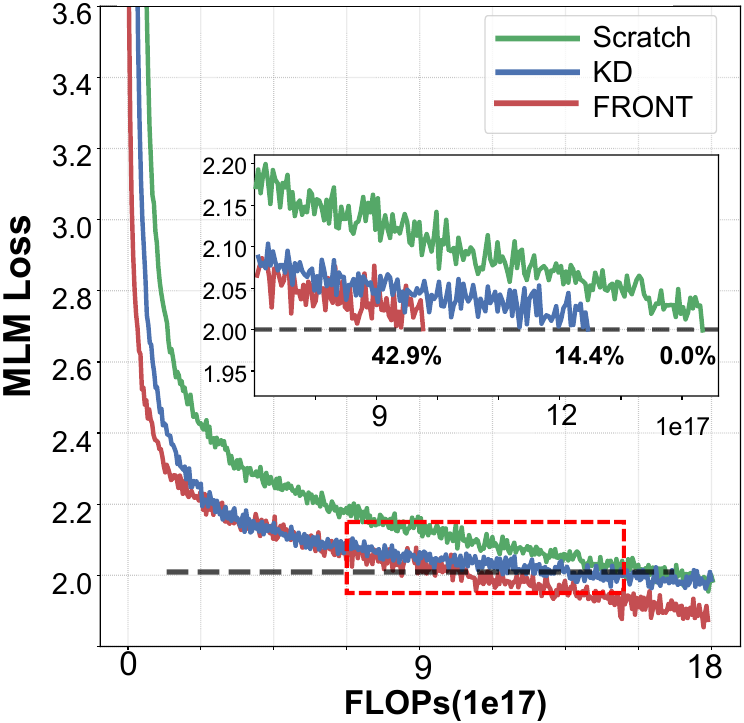}%
    }
    \hfill
    \subfloat[GPT-B $ \rightarrow$ GPT-S\label{fig:c1}]{%
        \includegraphics[width=0.333\textwidth]{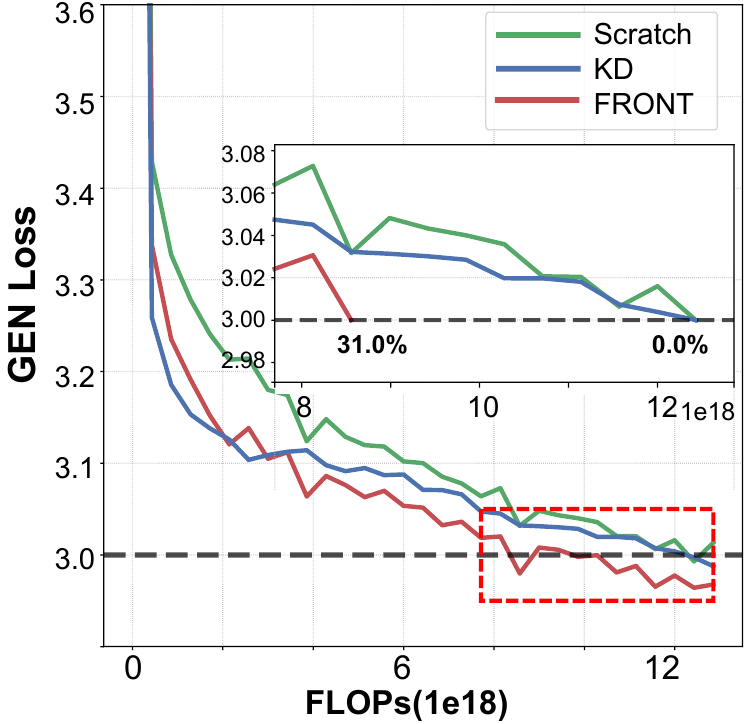}%
    }

    \caption{Results of pretraining BRET-S, RoBERTa-S and GPT-S. FRONT can achieve the highest savings in FLOPs with 47.5\% for BERT-S, 42.9\% for RoBERTa-S, and 31.0\% for GPT-S from the Scratch models.}
    \label{fig:llm1}
    \vspace{-0.14in}
\end{figure*}

%初始化在跨域任务中
% \textbf{Initialization on Cross-Domain Tasks.} To further verify the capability of FRONT in extracting generalizable knowledge, we follow the experimental protocols introduced in~\citep{fu2024cross} for object detection and~\citep{nie2024cross} for image segmentation from Source Domain (SD) to Target Domain (TD). 
% Specifically, we evaluate our approach on two architectures, including ViT-S/B and ResNet50, using six object detection datasets and four image segmentation datasets, comparing the \textsc{FRONT} initialization with the scratch training, other details in Appendix~\ref{app:dataset}.
% % 表1

% As shown in Table~\ref{tab:object_detection} and Table~\ref{tab:image_segmentation}, the proposed \textsc{FRONT} initialization substantially improves performance across the evaluated cross-domain tasks compared to training from scratch.
% Especially in image segmentation, FRONT achieves an average performance improvement of 18.26\%. 
% % These results demonstrate the effectiveness of extracting and transferring task-agnostic knowledge, leading to consistently higher accuracy in both object detection and image segmentation tasks.
% Unlike prior works limited to same-domain transfer, \textsc{FRONT} enables robust classification-to-dense generalization, effectively extracting and repurposing global classification knowledge for downstream dense prediction.

% 
\subsection{Results on Language Tasks}
\label{sec:llm}

\textbf{More Effective Pre-training.} Across all three architectures, models initialized with FRONT demonstrate a markedly faster convergence, as presented in Figure~\ref{fig:llm1}. This effect is particularly dramatic for BERT, where FRONT also provides a significantly lower starting loss, and this performance lead is maintained throughout the training. While KD also leverages pre-trained knowledge, it incurs a substantial, cumulative computational cost from repeated teacher inferences at every training step. In contrast, FRONT imparts its full knowledge benefit in a single, zero-overhead initialization, leading to significant computational savings. On average, our method reduces the required training FLOPs by 40.5\% across the three architectures compared to standard from-scratch training, underscoring its superiority as a resource-efficient knowledge transfer mechanism. 

FRONT's robust efficiency across diverse models stems from its architecture-agnostic nature. The semantic difference between bidirectional and causal attention lies in the forward-pass mask, not the weight matrices. Because their learnable parameters share identical dimensions and function as linear projectors, FRONT directly maps low-frequency coefficients between them.

%In downstream tasks of BERT，we set the batch size to 32 and use Adam with the learning rate 2e-5and epochs from {4, 5, 10} for the GLUE tasks。我们直接将初始化后的Small模型直接用来训练下游任务，知识蒸馏中需要将Base作为教师模型，结果如表所示。表明我们的方法在下游模型中具有很好的泛化能力，且能够提取出无监督掩码语言模型的基础能力，即预测masked的词的基础能力，并在下游具体任务中得到很好的适应性。结果表明我们方法能兼容多架构并取得一致性的好的初始化能力。

% \begin{figure}[tb]
% % \vspace{-0.14in}
%     \centering
%     \subfloat[Random Init.\label{fig:random_init}]{%
%         \includegraphics[width=0.33\textwidth]{img/visualization_random_init.pdf}%
%     }
%     \hfill
%     \subfloat[Direct Pre-train\label{fig:pretrained DeiT-T}]{%
%         \includegraphics[width=0.33\textwidth]{img/visualization_pretrained.pdf}%
%     }
%     \hfill
%     \subfloat[FRONT\label{fig:FRONT}]{%
%         \includegraphics[width=0.33\textwidth]{img/visualization_front.pdf}%
%     }
%     \vspace{-0.07in}
%     \caption{Self-attention layer structural patterns illustrated. The figure displays matrices $W_qW_k^T$ (left) and $W_vW_{o}$ (right) across three DeiT-T configurations: random initialization, pre-trained, and FRONT-initialized models. FRONT can inherit the diagonal property of self-attention layers that only exists in pre-trained ViTs.}
%     \label{fig:self_attention}
%     \vspace{-0.17in}
% \end{figure}

\textbf{Generalization on GLUE Benchmark.} As presented in Table~\ref{tab:glue_results1}, the model initialized by FRONT significantly outperforms both the scratch and knowledge distillation baselines across all GLUE tasks. These results demonstrate that our method possesses strong generalization capabilities. It effectively transfers the core competencies acquired during pre-training—such as the fundamental ability to predict masked tokens, which readily adapt to a diverse set of specific downstream tasks. Overall, these findings confirm that FRONT provides a robust and consistent initialization method that is compatible with multiple model architectures.

% \begin{table}[htbp]
% \centering
% \vspace{-0.1in}
% \caption{GLUE benchmark results on BERT-S.}
% \vspace{-0.1in}
% \label{tab:glue_results1}
% % \resizebox{<宽度>}{<高度>}{内容}，! 表示高度自适应
% \resizebox{0.82\textwidth}{!}{%
% \begin{tabular}{lcccccccc}
% \toprule
% Method & SST-2 & MNLI & MRPC & CoLA & QNLI & QQP & STS-B & Average \\
% \midrule
% Scratch & 78.97 & 63.65 & 66.42 & 8.14 & 57.21 & 81.50 & 17.27 & 53.31 \\
% Knowledge Distillation & 77.75 & 64.81 & 64.14 & 9.14 & 58.00 & 80.07 & 15.76 & 52.81 \\
% FRONT{} & \textbf{82.44} & \textbf{74.08} & \textbf{68.46} & \textbf{17.38} & \textbf{70.86} & \textbf{84.93} & \textbf{54.23} & \textbf{64.63} \\
% \bottomrule
% \end{tabular}%
% }
% \vspace{-0.1in}
% \end{table}

\subsection{Ablation and Analysis}
\label{sec:ablation}

% Further analysis of FRONT-initialized models indicates that low-frequency components mainly encode generalizable knowledge, while high-frequency components are more likely to capture task-specific details (see Appendix~\ref{app:freq_init}).

% \begin{wraptable}{r}{0.38\textwidth}
%     \centering
%     \setlength{\tabcolsep} {2 mm}
%     \vspace{-0.15in}
%     \caption{Knowledge transfer across diverse frequency components. Low-frequency knowledge surpasses others.}
%     \vspace{-0.1in}
%     \resizebox{0.38 \textwidth}{!}{
%     \begin{tabular}{@{}l|ccccc@{}}
%       \toprule
%       \multirow{2}{*}{\textbf{\shortstack{Select\\Freq.}}} & \multicolumn{5}{c}{\textbf{DeiT-Ti}} \\
%       \cmidrule{2-6}
%       & \textbf{L4} & \textbf{L6} & \textbf{L8} & \textbf{L10} & \textbf{L12} \\
%       \midrule
%       He-Init & 34.7 & 40.6 & 43.7 & 46.8 & 48.3\\
%       \midrule
%       Last-6 & 47.1 & 56.6 & 60.3 & 61.7 & 62.3 \\
%       Mid-6 & 47.4 & 62.2 & 64.3 & 64.6 & 65.2 \\
%       \cellcolor{blue!12}{First-6} & \cellcolor{blue!12}{\textbf{55.3}} & \cellcolor{blue!12}{\textbf{63.3}} & \cellcolor{blue!12}{\textbf{64.4}} & \cellcolor{blue!12}{\textbf{64.7}} & \cellcolor{blue!12}{\textbf{65.3}} \\
%       \bottomrule
%     \end{tabular}
%     }
%     \label{tab:freq_comparison}
%     \vspace{-0.16in}
% \end{wraptable}

\subsubsection{Effect of Frequency-Ratio and Decay Factor}
\label{sec:ablation_r_gamma}

\begin{table}

\centering
\setlength{\tabcolsep}{3.8 mm}
\caption{Top-1 accuracy (\%) on ImageNet-1K as a function of $r$ and $\gamma_d$. And `--' corresponds to a direct application of FRONT without refinement.}
\vspace{-0.1in}
\resizebox{0.48\textwidth}{!}{
\begin{tabular}{c|ccccc}
\toprule
$\gamma_d \setminus r$ & 0.33 & 0.50 & 0.67 & 0.83 & 1.00 \\
\midrule
--       & 60.63 & 62.74 & 64.79 & 65.06 & 66.19 \\
$\nicefrac{1}{16}$ & 66.61 & 66.64 & 66.65 & 66.58 & 66.60 \\
$\nicefrac{1}{8}$  & 66.97 & 67.22 & 67.02 & 67.11 & 67.13 \\
$\nicefrac{1}{4}$  & 67.08 & 67.58 & 68.02 & 67.92 & 67.88 \\
$\nicefrac{1}{2}$  & 66.74 & 66.92 & 66.00 & 66.72 & 66.94 \\
1      & 66.57 & 66.76 & 66.60 & 66.69 & 66.75 \\
\bottomrule
\end{tabular}
}
\label{tab:ablation_r_gamma}
\vspace{-0.22in}
\end{table}

To investigate the influence of the low-frequency retention ratio $r$ and the refinement decay factor $\gamma_d$, we perform a ablation study using ImageNet-1K classification on DeiT-Ti\_L8 as the source model on Tabel~\ref{tab:ablation_r_gamma}. The frequency ratio $r$ denotes the proportion of DCT coefficients retained as the \textit{learngene} from the source network parameters. 
The decay factor $\gamma_d$, as defined in Eq.~(~\ref{m}~), controls the suppression strength of high-frequency components during the refinement stage.

Across all $r$ values, the proposed FRONT consistently outperforms the \textit{Scratch} baseline by a large margin ($\geq$ 13.0+\%), confirming the robustness of low-frequency parameter transfer. 
Performance gains plateau when $r \in [0.50, 0.83]$ and $\gamma_d \in [0.125, 0.25]$, indicating that moderate refinement effectively attenuates harmful high-frequency noise without excessively discarding potentially useful fine-grained details. 
Extreme values of $r$ or $\gamma_d$ yield little additional benefit, suggesting that hyperparameter tuning can be coarse-grained for practical deployment in resource-limited environments.

\begin{table*}[htbp]
\centering
\setlength{\tabcolsep}{5.2mm}
\caption{Performance comparison of varying module configurations on BERT-B to BERT-S across GLUE benchmarks. The transferable low-frequency knowledge in language models is primarily concentrated in the self-attention weights ($W_{qkv}$ and $W_o$).}
\label{tab:bert_ablation}
\vspace{-0.07in}
\resizebox{0.94\linewidth}{!}{
\begin{tabular}{cccccccccccc}
\toprule
\textbf{$W_{qkv}$} & \textbf{$W_o$} & \textbf{$W_{mlp}$} & \textbf{norm} & \textbf{SST-2} & \textbf{MNLI} & \textbf{MRPC} & \textbf{CoLA} & \textbf{QNLI} & \textbf{QQP} & \textbf{STS-B} & \textbf{Avg} \\
\midrule
 & & & & 78.97 & 63.65 & 66.42 &  8.14 & 57.21 & 81.50 & 17.27 & 53.31 \\
\checkmark & & & & 82.11 & 72.14 & 68.62 &  4.63 & 66.95 & 84.71 & 43.19 & 60.34 \\
 & \checkmark & & & 82.34 & 66.70 & 68.87 & 10.23 & 62.57 & 81.81 & 42.79 & 59.33 \\
 & & \checkmark & & 80.73 & 67.29 & 69.61 & 12.34 & 60.97 & 78.78 & 36.85 & 58.08 \\
 & & & \checkmark & 80.85 & 65.69 & 67.40 & 14.35 & 61.76 & 81.50 & 42.65 & 59.17 \\
\checkmark & \checkmark & & & 81.19 & 72.17 & 70.34 & 16.42 & 69.98 & 84.82 & 50.35 & {\textit{63.61}} \\
\checkmark & \checkmark & \checkmark & \checkmark & 82.44 & 74.08 & 68.46 & 17.38 & 70.86 & 84.93 & 54.23 & {\textbf{64.63}} \\
\bottomrule
\end{tabular}
}
\vspace{-0.17in}
\end{table*}

\subsubsection{Balance between Duration and Regularization}

\label{sec:refine}

% \begin{figure}[ht]
% \centering

% \includegraphics[width=0.38\textwidth]{img/regularization_and_epoch.pdf}
% \vspace{-0.08in}
% \caption{Achieving Superior Performance with Efficient Refinement.} % 图片标题
% \vspace{-0.17in}
% \label{fig:re_e}
% \end{figure}

To show the ability of FRONT{\scalebox{0.8}{\text{++}}} strategy, we conduct a grid search by fine-tuning a DeiT-Ti-L8 source model on ImageNet-1K, exploring the interplay between fine-tuning duration across various epochs \{10, 20, 50, 100\} and regularization strength $\lambda$ values \{1e-5, 5e-4, 1e-3, 2e-3, 5e-3\}. Subsequently, the effectiveness of the 0.8M-parameter learngenes generated from each configuration is evaluated to initialize a deeper DeiT-Ti-L10 target model.

\begin{table}[htbp]
\centering
\caption{Module-level ablation study on DeiT-Ti (L12 to L4) evaluated on ImageNet-1K. The results indicate that transferable low-frequency knowledge is predominantly concentrated in the MLP weights for vision models.}
\vspace{-0.04in}
\label{tab:deit_ablation}
\begin{tabular}{ccccc}
\toprule
\setlength{\tabcolsep}{10.8mm}
\textbf{$W_{qkv}$} & \textbf{$W_o$} & \textbf{$W_{mlp}$} & \textbf{norm} & \textbf{Acc (\%)} \\
\midrule
 & & & & 34.7 \\
\checkmark & & & & 46.5 \\
 & \checkmark & & & 46.6 \\
 & & \checkmark & & {\textit{54.8}} \\
 & & & \checkmark & 47.7 \\
\checkmark & \checkmark & & & 50.1 \\
\checkmark & \checkmark & \checkmark & \checkmark & {\textbf{55.3}} \\
\bottomrule
\end{tabular}
\vspace{-0.15in}
\end{table}

The results, presented in Figure~\ref{fig:img_comb}~(a), reveal the remarkable efficiency of this approach. A mere 20 epochs of fine-tuning yields an initialization capability comparable to the 150-epoch from-scratch FRONT{\scalebox{1}{\text{+}}} variant, drastically reducing the required computation while still outperforming high-effort methods like WAVE. A clear trade-off also exists for $\lambda$: both insufficient (e.g., 1e-5) and excessive (e.g., 5e-3) regularization degrade performance, indicating the need for a balance between concentrating core knowledge and preserving useful high-frequency details. These findings establish FRONT{\scalebox{0.8}{\text{++}}} as a highly practical, resource-efficient solution for adapting pre-trained knowledge, achieving superior efficacy as a low-effort alternative.
% The results, presented in Figure~\ref{fig:re_e}, highlight the remarkable efficiency and effectiveness of our fine-tuning strategy, FRONT{\scalebox{0.8}{\text{++}}}. A key finding is that fine-tuning a pre-trained model for a mere 20 epochs achieves an initialization capability comparable to that of FRONT{\scalebox{1}{\text{+}}} trained for 150 epochs, demonstrating a dramatic reduction in the computational cost required to generate high-quality learngenes. Increasing the fine-tuning duration generally improves learngene quality, yet even our best-performing models require significantly less computation than competing retraining methods like WAVE. However, the regularization strength ($\lambda$) presents a clear trade-off. Both excessively low (e.g., 1e-5) and overly aggressive (e.g., 5e-3) values degrade performance. This suggests that an optimal balance is required: enough regularization to concentrate core knowledge, but not so much that it erases useful high-frequency details. These findings collectively demonstrate that FRONT{\scalebox{0.8}{\text{++}}} achieves superior initialization efficacy at a fraction of the computational cost of alternative high-effort retraining methods, positioning it as a highly practical solution for adapting pre-trained knowledge under constrained computational resources.
\begin{figure}
    \centering
    \includegraphics[width=\linewidth]{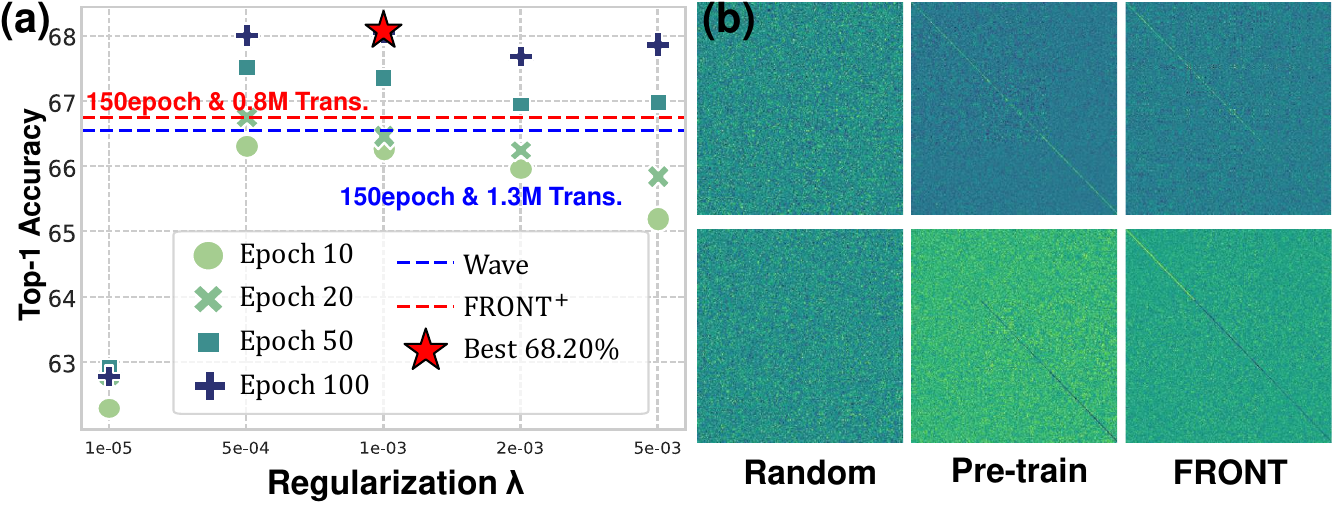}
    \caption{(a)Achieving superior performance with efficient refinement. (b) Self-attention layer structural patterns illustrated. The figure displays matrices $W_qW_k^T$ (top) and $W_vW_{o}$ (bottom) across three DeiT-T configurations: random initialization, pre-trained, and FRONT-initialized models. FRONT can inherit the diagonal property of self-attention layers that only exists in pre-trained ViTs.}
    \vspace{-0.16in}
    \label{fig:img_comb}
\end{figure}

\subsubsection{Module-Level Ablations}
\label{sec:module_ablations}

To further investigate the localization of transferable low-frequency knowledge within Transformer architectures, we conduct comprehensive module-level ablations across both vision and language domains. Specifically, we selectively apply our proposed FRONT initialization to distinct components: attention projections (\textbf{$W_{qkv}$}), attention output projections (\textbf{$W_o$}), feed-forward networks (\textbf{$W_{mlp}$}), and normalization layers (\textbf{norm}). Modules not selected for FRONT transfer are randomly initialized. 

We evaluate this on DeiT-Ti-L12 to DeiT-Ti-L4 transfer on ImageNet-1K for the vision domain, and BERT-B to BERT-S transfer on the GLUE benchmark for the language domain. The results reveal a fascinating divergence in how different modalities encode structural knowledge. As shown in Table~\ref{tab:deit_ablation}, for DeiT, the transferable low-frequency knowledge is predominantly concentrated in the MLP weights. Applying FRONT solely to the MLP blocks achieves an accuracy of 54.8\%, which significantly outperforms transferring any other single module and closely approaches the optimal performance of 55.3\% achieved when transferring all modules. Conversely, Table~\ref{tab:bert_ablation} demonstrates that for BERT, the transferable knowledge is heavily concentrated in the self-attention mechanisms. Transferring only the self-attention weights ($W_{qkv}$ and $W_o$) yields an average GLUE score of 63.61, substantially higher than transferring the MLP weights alone (58.08) and nearly matching the full-transfer performance of 64.63.

\subsubsection{Visualization of Structural Knowledge}

% \begin{figure*}[tb]
% % \vspace{-0.14in}
%     \centering
%     \subfloat[Random Init.\label{fig:random_init}]{%
%         \includegraphics[width=0.33\textwidth]{img/visualization_random_init.pdf}%
%     }
%     \hfill
%     \subfloat[Direct Pre-train\label{fig:pretrained DeiT-T}]{%
%         \includegraphics[width=0.33\textwidth]{img/visualization_pretrained.pdf}%
%     }
%     \hfill
%     \subfloat[FRONT\label{fig:FRONT}]{%
%         \includegraphics[width=0.33\textwidth]{img/visualization_front.pdf}%
%     }
%     \vspace{-0.07in}
%     \caption{Self-attention layer structural patterns illustrated. The figure displays matrices $W_qW_k^T$ (left) and $W_vW_{o}$ (right) across three DeiT-T configurations: random initialization, pre-trained, and FRONT-initialized models. FRONT can inherit the diagonal property of self-attention layers that only exists in pre-trained ViTs.}
%     \label{fig:self_attention}
%     \vspace{-0.17in}
% \end{figure*}

As illustrated in Figure~\ref{fig:img_comb}~(b), FRONT enables models to inherit the essential diagonal attributes within self-attention layers—a characteristic typically exclusive to pre-trained models. FRONT autonomously encapsulates structured knowledge from pre-trained models into the low-frequency domain without requiring manual intervention. Through the transmission of minimal parameters via DCT and IDCT operations, the initialized model inherently preserves these critical structural features in its parameter matrices, resulting in improved training efficiency and faster convergence.

\section{Conclusion}
% We propose FRONT, a novel initialization method that leverages knowledge from large pre-trained models to initialize networks of varying architectures.  With no extra cost, this approach improves model accuracy and reduces training time. By implementing frequency-based regularization during retraining, we concentrate knowledge representation primarily in low-frequency components. FRONT demonstrates superior performance when initializing models for both depth and width expansions, while the size-agnostic knowledge it transfers proves sufficiently generalizable for application across diverse downstream datasets.

% This paper empirically reveals that a model's ``learngene", where its core and transferable knowledge is encoded primarily in the low-frequency components of its weights, providing a robust and efficient mechanism for inheritance by downstream models. 
This paper empirically reveals that a model's ``learngene", where its core and transferable knowledge is encoded primarily in the low-frequency components of its weights, provides a robust and efficient mechanism for inheritance by downstream models.
We propose FRONT, a novel initialization framework that enables flexible and efficient reuse of parameters from a single pre-trained model to initialize diverse architectures with varying sizes. By leveraging the DCT, FRONT efficiently extracts the task-agnostic, low-frequency knowledge, which we identify as learngenes. Furthermore, we introduce an optional yet powerful frequency-based regularization strategy, FRONT{\scalebox{1}{\text{+}}}, which refines learngenes to enhance their generalizability by penalizing high-frequency artifacts.
% By transferring low-frequency coefficients–identified as learngenes, via DCT, facilitates resource-efficient and size-adaptive initialization without requiring repeated pretraining. 
% The core mechanism of FRONT is the transfer of low-frequency coefficients, which we identify as learngenes, via the DCT. This process facilitates a resource-efficient and size-adaptive initialization that avoids the need for repeated pre-training.
% The proposed frequency-based regularization further enhances the generalizability of learngenes, namely FRONT{\scalebox{1}{\text{+}}}, by promoting shared, transferable task-agnostic knowledge. 
% Empirical results demonstrate that our method consistently achieves superior accuracy and significantly reduced training time when initializing models of varying depths and widths, while the size-agnostic representation it leverages proves robust across a range of downstream tasks and datasets on visual and language domains.
Empirical results demonstrate that our method consistently achieves superior accuracy and significantly reduces training time when initializing models of varying depths and widths. Furthermore, the leveraged size-agnostic representation proves robust across diverse downstream tasks in both vision and language domains

\clearpage

\section*{Impact Statement}
The broader impact of FRONT lies in establishing a frequency-domain transfer paradigm that decouples foundational knowledge from rigid architectural constraints. By enabling the systematic reuse of pre-trained weights for "one-for-all" initialization, our approach significantly reduces the energy consumption and computational costs (e.g., training FLOPs) associated with model adaptation without extra training cost. This contributes to the goals of Green AI and lowers the barrier to entry for researchers and practitioners in resource-constrained environments, thereby facilitating the democratization of large-scale pre-trained models.

\section*{Acknowledgement}
\label{sec:acknowledge}

% This research was supported by the Jiangsu Science Foundation (BG2024036, BK20243012, BK20241297), the National Science Foundation of China (62125602, 62406066, U24A20324, 92464301, 625B2045), the New Cornerstone Science Foundation through the XPLORER PRIZE, and the Fundamental Research Funds for the Central Universities (2242025K30024).

This research was supported by the Jiangsu Science Foundation (BG2024036, BK20243012, BK20241297), the National Science Foundation of China (62406066, 62125602, U24A20324, 92464301), the New Cornerstone Science Foundation through the XPLORER PRIZE, and the Fundamental Research Funds for the Central Universities (2242025K30024), and the Big Data Computing Center of Southeast University.


\bibliography{example_paper}
\bibliographystyle{icml2026}

%%%%%%%%%%%%%%%%%%%%%%%%%%%%%%%%%%%%%%%%%%%%%%%%%%%%%%%%%%%%%%%%%%%%%%%%%%%%%%%
%%%%%%%%%%%%%%%%%%%%%%%%%%%%%%%%%%%%%%%%%%%%%%%%%%%%%%%%%%%%%%%%%%%%%%%%%%%%%%%
% APPENDIX
%%%%%%%%%%%%%%%%%%%%%%%%%%%%%%%%%%%%%%%%%%%%%%%%%%%%%%%%%%%%%%%%%%%%%%%%%%%%%%%
%%%%%%%%%%%%%%%%%%%%%%%%%%%%%%%%%%%%%%%%%%%%%%%%%%%%%%%%%%%%%%%%%%%%%%%%%%%%%%%
\newpage
\appendix
\onecolumn

\section{Analysis on FRONT}
\label{sec:appendix_theory}

We provide a formal analysis for the claim that low-frequency components of model weights encode task-invariant knowledge. We formulate the fine-tuning process as learning a target function composed of a shared, smooth component and a specific, high-frequency component. We prove that under standard spectral bias assumptions, the optimization gradients are concentrated in the high-frequency weight subspace.

Let $f(x; W): \mathcal{X} \to \mathcal{Y}$ be a neural network parameterized by $W \in \mathbb{R}^d$. Let $\hat{W} = \mathcal{D}W$. We partition the frequency spectrum $\mathcal{K}$ into low-frequency indices $\mathcal{I}_L$ and high-frequency indices $\mathcal{I}_H$. The weight space $\mathbb{R}^d$ is the direct sum of two orthogonal subspaces $\mathcal{W}_L \oplus \mathcal{W}_H$, where:
\[ W_L = \mathcal{D}^{-1}(M_L \odot \hat{W}) \in \mathcal{W}_L, \quad W_H = \mathcal{D}^{-1}(M_H \odot \hat{W}) \in \mathcal{W}_H \].
Here, $M_L$ and $M_H$ are binary masks. Due to the orthogonality of DCT, $\|W\|^2 = \|W_L\|^2 + \|W_H\|^2$.

\begin{assumption}
\label{ass:decomp}
The target function $f^*$ for a downstream task can be decomposed into a shared, generalizable component $f_G^*$ and a task-specific component $f_S^*$:
\[ f^*(x) = f_G^*(x) + f_S^*(x), \]
where $f_G^*$ captures common general knowledge and $f_S^*$ captures task-specific details.
\end{assumption}

\begin{assumption}
\label{ass:correspondence}
% Consistent with the spectral bias analysis in deep networks~\cite{rahaman2019spectral, xu2019frequency}, we assume a correspondence between the frequency spectrum of the weights and the Lipschitz smoothness of the function.
Motivated by the Spectral Bias of Deep Neural Networks~\citep{rahaman2019spectral, xu2019frequency}, which states networks learn low-frequency functions first, we posit that the functional smoothness is reflected in the structural smoothness of the weight matrices. 
Specifically, $f_G^*$ is $\lambda_L$-Lipschitz (smooth) and can be well-approximated by weights in $\mathcal{W}_L$ and $f_S^*$ is $\lambda_H$-Lipschitz (sharp), where $\lambda_H \gg \lambda_L$, requiring weights in $\mathcal{W}_H$ for accurate approximation.
\end{assumption}

\begin{assumption}
\label{ass:pretrain}
The pre-trained weights $W_0$ have converged to the generalizable component. That is:
\[ f(x; W_0) \approx f_G^*(x). \]
Consequently, the projection of $W_0$ onto the generalizable manifold is optimal, implying $\nabla_{W_L} \mathcal{L}_{task}(W_0) \approx 0$ if the task were purely $f_G^*$.
\end{assumption}

We then analyze the gradient when fine-tuning on the target task with loss function $\mathcal{L}(W) = \frac{1}{2} \mathbb{E}_{x} [\| f(x; W) - f^*(x) \|^2]$.

\begin{theorem}
\label{thm:gradient}
Under Assumptions \ref{ass:decomp}-\ref{ass:pretrain}, the gradient of the loss function at the onset of fine-tuning, $G = \nabla_W \mathcal{L}(W_0)$, satisfies:
\[ \| \mathcal{P}_{H}(G) \| \gg \| \mathcal{P}_{L}(G) \|, \]
where $\mathcal{P}_{H}$ and $\mathcal{P}_{L}$ are orthogonal projections onto the high-frequency ($\mathcal{W}_H$) and low-frequency ($\mathcal{W}_L$) subspaces, respectively.
\end{theorem}

\begin{proof}

The gradient of the loss $\mathcal{L}(W)$ with respect to weights $W$ is given by:
\[ G = \nabla_W \mathcal{L}(W_0) = \mathbb{E}_{x} \left[ \nabla_W f(x; W_0)^T \big( f(x; W_0) - f^*(x) \big) \right] \]

Using Assumption \ref{ass:pretrain} and Assumption \ref{ass:decomp}, the residual error term becomes:
\begin{equation*}
r(x) = f(x; W_0) - f^*(x) \approx f_G^*(x) - (f_G^*(x) + f_S^*(x)) = -f_S^*(x).
\end{equation*}
Thus, the optimization is driven solely by the negative task-specific component $-f_S^*(x)$.

We project the gradient $G$ onto the low-frequency subspace $\mathcal{W}_L$:
\[ \mathcal{P}_L(G) = \mathbb{E}_{x} \left[ \mathcal{P}_L \left( \nabla_W f(x; W_0) \right)^T (-f_S^*(x)) \right] \]
Here, $\nabla_W f(x; W_0)$ represents the Jacobian of the network. The term $\mathcal{P}_L (\nabla_W f)$ captures the sensitivity of the network output to changes in low-frequency weights.
Based on Assumption \ref{ass:correspondence}, functions generated by variations in $\mathcal{W}_L$ are smooth. Let $\phi_L(x)$ be a basis function in the function space spanned by varying $W_L$.
Conversely, $f_S^*(x)$ is a high-frequency (sharp) function.

In the frequency domain of the function space, the inner product between a smooth function (variations caused by $W_L$) and a high-frequency residual ($f_S^*$) is negligible due to spectral separation, i.e., $\langle \text{Smooth}, \text{Sharp} \rangle \approx 0$.
Therefore:
\[ \| \mathcal{P}_L(G) \| = \left\| \mathbb{E}_{x} [ \nabla_{W_L} f \cdot (-f_S^*) ] \right\| \approx 0. \]
This result aligns with the intuition that low-frequency weights cannot effectively fit high-frequency noise/details; thus, the gradient (which represents the ``desire" to fit the error) is small in this direction.

Consider the projection onto $\mathcal{W}_H$:
\[ \mathcal{P}_H(G) = \mathbb{E}_{x} \left[ \nabla_{W_H} f(x; W_0)^T (-f_S^*(x)) \right]. \]
Since variations in $W_H$ can generate high-frequency functions, the spectral support of $\nabla_{W_H} f$ overlaps significantly with that of $f_S^*$. The inner product is non-zero and significant:
\[ \| \mathcal{P}_H(G) \| \propto \| f_S^* \| \gg 0. \]
Then we have: $\| \mathcal{P}_H(G) \| \gg \| \mathcal{P}_L(G) \|$, which concludes the proof.

\end{proof}

Since the gradient for low-frequency components is negligible, these weights remain stable during fine-tuning, confirming their role as the immutable ``learngene''. In contrast, high-frequency weights require significant updates to fit $f_S^*$. FRONT explicitly preserves the stable $W_L$ while allowing $W_H$ to be re-initialized or adapted, thereby aligning the initialization with the spectral requirements of transfer learning.

\section{Preliminary}

\subsection{Vision Transformer Architecture}
\label{sec:vit_arch}

The Vision Transformer encoder establishes a mapping relationship:
\begin{equation}
    \mathcal{E}: \mathbb{R}^{n \times d} \to \mathbb{R}^{n \times d},
\end{equation}
composed of $L$ successive transformation layers, where each layer combines multi-head self-attention (MSA) operators with multi-layer perceptron (MLP) projections through residual connections.

The Multi-Head Self-Attention (MSA) mechanism allows the model to attend to different parts of the input sequence, capturing long-range dependencies between image patches. It consists of $h$ parallel self-attention heads. For each attention head $i \in \{1,...,h\}$, the self-attention head operation computes:

\begin{equation}
\mathsf{\text{Attn}}_i(\mathbf{X}) = \mathsf{\text{softmax}}\left(\frac{(\mathbf{X}\mathbf{W}^q_i)(\mathbf{X}\mathbf{W}^k_i)^\top}{\sqrt{d_h}} \right)(\mathbf{X}\mathbf{W}^v_i),
\label{eq:attention}
\end{equation}

\begin{equation}
\mathsf{\text{MSA}}(\mathbf{X}) = \left[ \mathsf{\text{Attn}}_1(\mathbf{X}); \cdots; \mathsf{\text{Attn}}_h(\mathbf{X}) \right] \mathbf{W}_o \quad \text{with} \quad \mathbf{W}_o \in \mathbb{R}^{hd_h \times D},
\label{eq:msa}
\end{equation}

where $\mathbf{W}^q_i, \mathbf{W}^k_i, \mathbf{W}^v_i \in \mathbb{R}^{D \times d_h}$ are learnable projection matrices, and $d_h = D/h$ ensures dimensional consistency. These projections can be compactly represented through concatenation as $\mathbf{W}_{qkv} \in \mathbb{R}^{D \times 3hd}$ where $3hd = 3h \cdot d_h$ corresponds to the combined query-key-value dimensions. 

The position-wise MLP implements feature transformation through dimension expansion:

\begin{equation}
\mathsf{\text{MLP}}(\mathbf{X}) = \mathsf{\text{GELU}}(\mathbf{X}\mathbf{W}_{\text{in}} + \mathbf{b}_1)\mathbf{W}_{\text{out}} + \mathbf{b}_2,
\label{eq:mlp}
\end{equation}

where $\mathbf{W}_{\text{in}} \in \mathbb{R}^{D \times D'}$ and $\mathbf{W}_{\text{out}} \in \mathbb{R}^{D' \times D}$ with expansion ratio $\rho = D'/D$ typically set to 4.

The Vision Transformer (ViT) first divides an input image into $N$ patches of size $P \times P$. Each patch is then linearly projected into a $D$-dimensional embedding. To retain positional information, a positional encoding $\mathbf{E} \in \mathbb{R}^{N \times D}$ is added to the patch embeddings, resulting in the input sequence $\mathbf{X} \in \mathbb{R}^{N \times D}$. This sequence is then processed through $L$ successive encoder layers, each implementing dual transformation pathways:

\begin{equation}
\mathsf{\text{EncoderLayer}}(\mathbf{X}) = \mathsf{\text{MLP}}(\mathsf{\text{LayerNorm}}(\mathsf{\text{MSA}}(\mathsf{\text{LayerNorm}}(\mathbf{X})))) + \mathbf{X},
\end{equation}

where $\mathsf{\text{MSA}}$ denotes multi-head self-attention and $\mathsf{\text{MLP}}$ represents the position-wise feed-forward network. Layer Normalization ($\mathsf{\text{LayerNorm}}$) is applied before both $\mathsf{\text{MSA}}$ and $\mathsf{\text{MLP}}$ to stabilize training. This residual structure enables stable gradient flow during backpropagation.

\subsection{Discrete Cosine Transform (DCT) and Inverse DCT (IDCT)}
\label{dctandIDCT}

\paragraph{1D-DCT} For $x$ of length $M$, its transformed frequency domain representation $X$ also has a length of $M$:
\begin{equation}
X[k] = \alpha(k)\sum_{n=0}^{M-1} x[n] \cos\left(\frac{\pi(2n+1)k}{2M}\right),\quad k=0,...,M-1,
\end{equation}
where the scaling factor:
\begin{equation}
\label{eq:sample}
\alpha(k) = 
\begin{cases}
\sqrt{\dfrac{1}{M}}, & k=0 \\
\sqrt{\dfrac{2}{M}}, & k \geq 1
\end{cases}
\text{.}
\end{equation}

\paragraph{2D-DCT} For an input matrix $\mathbf{x} \in \mathbb{R}^{M \times N}$, the transformed coefficient matrix $\mathbf{X} \in \mathbb{R}^{M \times N}$ preserves the original dimensions through:

\begin{equation}
X[k,l] = \alpha(k)\alpha(l)\sum_{m=0}^{M-1}\sum_{n=0}^{N-1} x[m,n] \cos\left(\frac{\pi(2m+1)k}{2M}\right)\cos\left(\frac{\pi(2n+1)l}{2N}\right).
\end{equation}
This is foundational to JPEG compression, where $X[0,0]$ represents the DC (average intensity) component.

\paragraph{3D-DCT} For the weight ${x} \in \mathbb{R}^{M \times N \times P}$, the transformed coefficient matrix ${X} \in \mathbb{R}^{M \times N \times P}$ is:
\begin{equation}
X[k,l,q] = \mathcal{D}({x}) = \alpha(k)\alpha(l)\alpha(q) \sum_{m=0}^{M-1}\sum_{n=0}^{N-1}\sum_{p=0}^{P-1} x[m,n,p] \, C(m,n,p,k,l,q).
\end{equation}
\paragraph{3D-IDCT} The IDCT recovers the original signal:
\begin{equation}
x[m,n,p] = \mathcal{D}^{-1}({X}) = \alpha(k)\alpha(l)\alpha(q) \sum_{k=0}^{M-1}\sum_{l=0}^{N-1}\sum_{q=0}^{P-1} X[k,l,q] \, C(m,n,p,k,l,q),
\label{eq:3d_IDCT2}
\end{equation}
where $ 0 \leq m, k < M$, $0 \leq n, l < N$, $0 \leq p, q < P$, and $\alpha(d)$ follow the standard definition in 3D-DCT/3D-IDCT:
\begin{equation}
\begin{aligned}
    C(m,n,p,k,l,q) &= \cos\bigl(\frac{\pi(2m+1)k } {2M}\bigr) \cdot \cos\bigl(\frac{\pi(2n+1)l}{2N}\bigr) \cdot \cos\bigl(\frac{\pi(2p+1)q} {2P}\bigr), \\
\alpha(d) &= \begin{cases}
\sqrt{\frac{1}{S_d}} & \text{if } d=0 \\
\sqrt{\frac{2}{S_d}}  & \text{if } d>0
\end{cases}
\quad \text{with} \quad
S_d = \begin{cases}
M & \text{if } d=k \\
N & \text{if } d=l \\
P & \text{if } d=q
\end{cases}
\text{.}
\end{aligned}
\end{equation}

\paragraph{4D-DCT} For an input tensor $\mathbf{x} \in \mathbb{R}^{M \times N \times P \times Q}$, the transformed coefficient tensor $\mathbf{X} \in \mathbb{R}^{M \times N \times P \times Q}$ preserves the original dimensions through:

\begin{equation}
\begin{split}
X[k,l,p,q] = \alpha(k)\alpha(l)\alpha(p)\alpha(q) \sum_{m=0}^{M-1}\sum_{n=0}^{N-1}\sum_{o=0}^{P-1}\sum_{r=0}^{Q-1} x[m,n,o,r] \cos\left(\frac{\pi(2m+1)k}{2M}\right) \\
\cos\left(\frac{\pi(2n+1)l}{2N}\right) \cos\left(\frac{\pi(2o+1)p}{2P}\right) \cos\left(\frac{\pi(2r+1)q}{2Q}\right).
\end{split}
\end{equation}

\paragraph{N-Dimensional IDCT} For an input tensor $\mathbf{X} \in \mathbb{R}^{D_1 \times D_2 \times \dots \times D_N}$, the inverse transformed tensor $\mathbf{x} \in \mathbb{R}^{D_1 \times D_2 \times \dots \times D_N}$ is given by:

\begin{equation}
x[d_1, d_2, \dots, d_N] = \prod_{i=1}^{N} \alpha(d_i) \sum_{k_1=0}^{D_1-1} \sum_{k_2=0}^{D_2-1} \dots \sum_{k_N=0}^{D_N-1} X[k_1, k_2, \dots, k_N] \prod_{i=1}^{N} \cos\left(\frac{\pi(2d_i+1)k_i}{2D_i}\right),
\end{equation}
where
\begin{equation}
\alpha(d_i) = 
\begin{cases}
\sqrt{\dfrac{1}{D_i}}, & k_i=0 \\
\sqrt{\dfrac{2}{D_i}}, & k_i \geq 1
\end{cases}
\text{.}
\end{equation}

\section{FRONT Algorithms}
The following process takes the example of direct extraction (FRONT) and using knowledge distillation to train the auxiliary model from scratch (FRONT+).
\label{FRONT Algorithms}
\subsection{Application to Vision Transformers (ViTs)}
\label{app:vit_all}
\paragraph{Algorithm 1: ViTs Learngene Acquisition}
\label{app:vit_a1}
\textbf{Requirements:}
\begin{itemize}
    \item Pre-trained/Auxiliary model weight tensors $\mathbf{\Omega}_{\text{pre/aux}} = \{ \mathbf{W}^{(1\sim L)}_{qkv}, \mathbf{W}^{(1\sim L)}_{o}, \mathbf{W}^{(1\sim L)}_{fc1}, \mathbf{W}^{(1\sim L)}_{fc2} \} $, frequency ratio $r$.
    \item For Method 2: Training dataset $\{ (x^{(i)}, y^{(i)}) \}_{i=1}^m$, number of epochs $N_{ep}$, batch size $B$, learning rate $\alpha$, regularization weight $\lambda$, decay rates $\gamma_d$.
\end{itemize}

\textbf{Output}: Learngenes (Compact spectral representation).

\begin{enumerate} % Main Level (Method 1 and 2)
    \item \textbf{Method 1: Direct Truncation of pre-trained Weights' Spectra}:
    \begin{enumerate} % Level 2 (Steps in Method 1)
        \item Initialize an empty set for learngene $\mathcal{G}_1 = \emptyset$.
        \item For each weight tensor $\mathbf{W}_i \in \mathbf{\Omega}_{\text{pre}}$:
        \begin{enumerate} % Level 3 (Steps inside the loop)
            \item Apply 3D-DCT: 
            \begin{align*}
                \mathbf{\Phi}_i = \mathcal{D}(\mathbf{W}_i) \quad
            \end{align*}
            \item Construct binary truncation mask: 
            \begin{align*}
                \mathbf{M}_r[l,m,n] = 
                \begin{cases}
                    1 & \text{if } l \leq \lfloor rL \rfloor, m \leq \lfloor rd_{\text{in}} \rfloor, n \leq \lfloor rd_{\text{out}} \rfloor \\
                    0 & \text{otherwise}
                \end{cases}
                \quad \text{(Eq. (~\ref{eq:truncation}) )}
            \end{align*}
            \item Apply truncation: $\hat{\mathbf{\Phi}}_i = \mathbf{M}_r \odot \mathbf{\Phi}_i$.
            \item Add the non-zero elements of $\hat{\mathbf{\Phi}}_i$ to $\mathcal{G}_1$. (Learngenes from Method 1)
        \end{enumerate}
        \item Output $\mathcal{G}_1$ as Learngenes.
    \end{enumerate}

    \item \textbf{Method 2: Distillation-Aware Training with a high-frequency attenuation regularization from scratch}:
    \begin{enumerate} % Level 2 (Steps in Method 2)
        \item Initialize Auxiliary Model with weights $\mathbf{\Omega}_{\text{aux}}$ (copying $\mathbf{\Omega}_{\text{pre}}$ or random, where we use random Initialization).
        \item Initialize training optimizer for $\mathbf{\Omega}_{\text{aux}}$.
        \item For $ep = 1$ to $N_{ep}$:
        \begin{enumerate} % Level 3 (Steps inside the outer loop)
            \item For each batch $\{ (x_j, y_j) \}_{j=1}^B$:
            \begin{enumerate} % Level 4 (Steps inside the batch loop)
                \item Get logits: $z_{\text{pre},j} = f_{\text{pre}}(x_j)$, $z_{\text{aux},j} = f_{\text{aux}}(x_j)$.
                \item Calculate Logit Loss: 
                \begin{align*}
                    \mathcal{L}_{\text{logit}} = \frac{1}{B}\sum_{j=1}^B \left( \text{KL}(z_{\text{pre},j} \| z_{\text{aux},j}) + \text{CE}(z_{\text{aux},j}, y_j) \right)
                \end{align*}
                \item Calculate the high-frequency attenuation regularization Loss $\mathcal{L}_{\text{reg}}$: 
                
                \hspace{1em} $\bullet$ Set $\mathcal{L}_{\text{reg}} = 0$.
                
                \hspace{1em} $\bullet$ For each weight tensor $\mathbf{W}_{\text{aux}, i} \in \mathbf{\Omega}_{\text{aux}}$:
                
                \hspace{2em} $\circ$ Apply 3D-DCT: $\mathbf{\Phi}_{\text{aux}, i} = \mathcal{D}(\mathbf{W}_{\text{aux}, i})$.
                
                \hspace{2em} $\circ$ Construct adaptive mask $\mathbf{M}_r[l,m,n]$ using Eq. (~\ref{m}).
                
                \hspace{2em} $\circ$ Calculate term: 
                \begin{align*}
                    Term_i = \frac{1}{\|\mathbf{M}_r \odot \mathbf{\Phi}_{\text{aux}, i}\|_0} \sum_{l,m,n} (\mathbf{M}_r[l,m,n] \cdot \mathbf{\Phi}_{\text{aux}, i}[l,m,n])^2 \quad \text{(Eq. (~\ref{reg}) )}
                \end{align*}
                
                \hspace{2em} $\circ$ $\mathcal{L}_{\text{reg}} = \mathcal{L}_{\text{reg}} + Term_i$.
                
                \item Calculate Total Loss: 
                \begin{align*}
                    \mathcal{L}_{\text{total}} = (1-\lambda)\mathcal{L}_{\text{logit}} + \lambda\mathcal{L}_{\text{reg}} \quad \text{(Eq. (~\ref{eq:total_loss}) )}
                \end{align*}
                \item Perform backward pass and update weights: 
                \begin{align*}
                    \mathbf{\Omega}_{\text{aux}} \leftarrow \mathbf{\Omega}_{\text{aux}} - \alpha \nabla \mathcal{L}_{\text{total}}
                \end{align*}
            \end{enumerate}
        \end{enumerate}
        \item Initialize $\mathcal{G}_2 = \emptyset$.
        \item For each trained weight tensor $\mathbf{W}_{\text{aux}, i} \in \mathbf{\Omega}_{\text{aux}}$:
        \begin{enumerate} % Level 3 (Steps inside the loop)
            \item Apply 3D-DCT: $\mathbf{\Phi}_{\text{aux}, i} = \mathcal{D}(\mathbf{W}_{\text{aux}, i})$.
            \item Add low-frequency components of $\mathbf{\Phi}_{\text{aux}, i}$ to $\mathcal{G}_2$. (learngenes from Method 2)
        \end{enumerate}
        \item Output $\mathcal{G}_2$ as learngenes.
    \end{enumerate}
\end{enumerate}

\paragraph{Algorithm 2: ViTs Initialization of Variable-Sized Models}
\label{app:vit_a2}
\textbf{Input:} Acquired Learngenes $\mathcal{G} = \{ \hat{\mathbf{\Phi}}_i \}$, target model dimensions $L'$, $d'_{\text{in}}$, $d'_{\text{out}}$.

\textbf{Output:} Initialized weight tensors for the variable-sized target model $\mathbf{\Omega}' = \{ \mathbf{W}'^{(1\sim L')}_{qkv}, \mathbf{W}'^{(1\sim L')}_{o}, \mathbf{W}'^{(1\sim L')}_{{fc1}}, \mathbf{W}'^{(1\sim L')}_{{fc2}} \}$.

\begin{enumerate}[label=\arabic*.]
    \item \textbf{Spectral Transformation:}
    \begin{enumerate}[label=\alph*.]
        \item \textbf{Case 1: $L \leq L'$, $d_{\text{in}} \leq d'_{\text{in}}$, $d_{\text{out}} \leq d'_{\text{out}}$ (Padding):}
        \begin{enumerate}[label=\roman*.]
            \item Apply zero-padding to $\hat{\mathbf{\Phi}}_i$ to obtain $\mathbf{\Phi}'_i \in \mathbb{R}^{L' \times d'_{\text{in}} \times d'_{\text{out}}}$.  Pad along each dimension as needed.
        \end{enumerate}
        \item \textbf{Case 2: $L > L'$, $d_{\text{in}} > d'_{\text{in}}$, $d_{\text{out}} > d'_{\text{out}}$ (Truncation):}
        \begin{enumerate}[label=\roman*.]
            \item Truncate $\hat{\mathbf{\Phi}}_i$ to obtain $\mathbf{\Phi}'_i \in \mathbb{R}^{L' \times d'_{\text{in}} \times d'_{\text{out}}}$. Discard high-frequency elements along each dimension as needed.
        \end{enumerate}
        \item \textbf{Case 3: Mixed Padding and Truncation:}
        \begin{enumerate}[label=\roman*.]
            \item Apply padding or truncation along each dimension ($L$, $d_{\text{in}}$, $d_{\text{out}}$) as needed to obtain $\mathbf{\Phi}'_i \in \mathbb{R}^{L' \times d'_{\text{in}} \times d'_{\text{out}}}$.
        \end{enumerate}
    \end{enumerate}
    \item \textbf{Weight Reconstruction:}
    \begin{enumerate}[label=\alph*.]
        \item Apply IDCT:
        \begin{align*}
            \mathbf{W}'_i = \mathcal{D}^{-1}(\mathbf{\Phi}'_i) \in \mathbb{R}^{L' \times d'_{\text{in}} \times d'_{\text{out}}}
        \end{align*}
        \item Add the reconstructed weight tensor $\mathbf{W}'_i$ to the set $\mathbf{\Omega}'$.
        \item Output $\mathbf{\Omega}'$ as the initialized target model weights.
        
    \end{enumerate}
\end{enumerate}

\subsection{Application to Multi-Layer Perceptrons (MLPs)}
\label{app:front_mlp}

The general framework described in Algorithms~\ref{app:mlp_a1} and~\ref{app:mlp_a2} can be applied to Multi-Layer Perceptrons (MLPs) by adapting the dimensionality of the spectral transformation to match the weight tensors of MLP layers. MLP layers typically have 2D weight matrices $\mathbf{W} \in \mathbb{R}^{D_{\text{in}} \times D_{\text{out}}}$, where $D_{\text{in}}$ and $D_{\text{out}}$ are the input and output dimensions of the layer, respectively. Biases are typically 1D vectors and are handled separately by 1D-DCT/1D-IDCT and not included in the following analysis. It is important to note that arbitrary transformations can be performed on the MLP in terms of width or depth. In the case of width changes as columns, the 2D-DCT/2D-IDCT needs to be applied.

\paragraph{Algorithm 1: MLP Learngene Acquisition}
\label{app:mlp_a1}
Applying Algorithm~\ref{app:mlp_a1} for Learngene Acquisition to MLPs follows the same two methods, but utilizes the 2D Discrete Cosine Transform ($\mathcal{D}_{2D}$) and its inverse ($\mathcal{D}^{-1}_{2D}$).

\textbf{Requirements:}
\begin{itemize}
    \item Pre-trained/Auxiliary MLP model weight matrices $\mathbf{\Omega}_{\text{pre/aux}} = \{ \mathbf{W}^{(l)} \}_{l=1}^L $, frequency ratio $r$. Here, each $\mathbf{W}^{(l)} \in \mathbb{R}^{D^{(l)}_{\text{in}} \times D^{(l)}_{\text{out}}}$.
    \item For Method 2: Training dataset $\{ (x^{(i)}, y^{(i)}) \}_{i=1}^m$, number of epochs $N_{ep}$, batch size $B$, learning rate $\alpha$, regularization weight $\lambda$. (Decay rates $\gamma_d$ if applicable).
\end{itemize}

\textbf{Output}: Learngenes (Compact 2D spectral representation).

\begin{enumerate} % Main Level (Method 1 and 2)
    \item \textbf{Method 1: Direct Truncation of pre-trained Weights' Spectra}:
    \begin{enumerate} % Level 2 (Steps in Method 1)
        \item Initialize an empty set for learngene $\mathcal{G}_1 = \emptyset$.
        \item For each weight matrix $\mathbf{W}_i \in \mathbf{\Omega}_{\text{pre}}$:
        \begin{enumerate} % Level 3 (Steps inside the loop)
            \item Apply 2D-DCT: $\mathbf{\Phi}_i = \mathcal{D}_{2D}(\mathbf{W}_i)$.
            \item Construct binary truncation mask: Adapt Eq. (~\ref{eq:truncation}) for 2D, using indices $(l,m)$ and dimensions $(D_{\text{in},i}, D_{\text{out},i})$ of $\mathbf{W}_i$. The mask $\mathbf{M}_r[l,m]$ is 1 if $l \leq \lfloor r D_{\text{in},i} \rfloor$ and $m \leq \lfloor r D_{\text{out},i} \rfloor$, and 0 otherwise.
            \item Apply truncation: $\hat{\mathbf{\Phi}}_i = \mathbf{M}_r \odot \mathbf{\Phi}_i$.
            \item Add the non-zero elements of $\hat{\mathbf{\Phi}}_i$ to $\mathcal{G}_1$.
        \end{enumerate}
        \item Output $\mathcal{G}_1$ as Learngenes.
    \end{enumerate}

    \item \textbf{Method 2: Distillation-Aware Training with a high-frequency attenuation regularization}:
    \begin{enumerate} % Level 2 (Steps in Method 2)
        \item Initialize Auxiliary MLP Model with weights $\mathbf{\Omega}_{\text{aux}}$ (as in Alg.~\ref{app:vit_a1}, Method 2 for ViT).
        \item Initialize training optimizer for $\mathbf{\Omega}_{\text{aux}}$.
        \item For $ep = 1$ to $N_{ep}$: Steps for batch processing, logit loss $\mathcal{L}_{\text{logit}}$, total loss $\mathcal{L}_{\text{total}}$, and weight updates are analogous to those in Alg.~\ref{app:vit_a1}, Method 2 for ViT.
        \begin{enumerate} % Level 3 (Steps inside the outer loop)
             \item For each batch $\{ (x_j, y_j) \}_{j=1}^B$:
             \begin{enumerate} % Level 4 (Steps inside the batch loop)
                \item Get logits $z_{\text{pre},j}, z_{\text{aux},j}$ and calculate $\mathcal{L}_{\text{logit}}$ (as in Alg.~\ref{app:vit_a1}, Method 2 for ViT).
                \item Calculate $\mathcal{L}_{\text{reg}}$:
                \begin{itemize}
                    \item Set $\mathcal{L}_{\text{reg}} = 0$.
                    \item For each weight matrix $\mathbf{W}_{\text{aux}, i} \in \mathbf{\Omega}_{\text{aux}}$:
                    \begin{itemize}
                        \item Apply 2D-DCT: $\mathbf{\Phi}_{\text{aux}, i} = \mathcal{D}_{2D}(\mathbf{W}_{\text{aux}, i})$.
                        \item Construct adaptive mask $\mathbf{M}_r[l,m]$ using adapted Eq. (~\ref{m}) for 2D dimensions.
                        \item Calculate term: Adapt Eq. (~\ref{reg}) for 2D dimensions:
                        \begin{align*}
                            Term_i = \frac{1}{\|\mathbf{M}_r \odot \mathbf{\Phi}_{\text{aux}, i}\|_0} \sum_{l,m} (\mathbf{M}_r[l,m] \cdot \mathbf{\Phi}_{\text{aux}, i}[l,m])^2
                        \end{align*}
                        \item $\mathcal{L}_{\text{reg}} = \mathcal{L}_{\text{reg}} + Term_i$.
                    \end{itemize}
                \end{itemize}
                \item Calculate Total Loss $\mathcal{L}_{\text{total}} = (1-\lambda)\mathcal{L}_{\text{logit}} + \lambda\mathcal{L}_{\text{reg}}$.
                \item Perform backward pass and update weights $\mathbf{\Omega}_{\text{aux}}$ (as in Alg.~\ref{app:vit_a1}, Method 2 for ViT).
             \end{enumerate}
        \end{enumerate}
        \item Initialize $\mathcal{G}_2 = \emptyset$.
        \item For each trained weight matrix $\mathbf{W}_{\text{aux}, i} \in \mathbf{\Omega}_{\text{aux}}$:
        \begin{enumerate} % Level 3 (Steps inside the loop)
            \item Apply 2D-DCT: $\mathbf{\Phi}_{\text{aux}, i} = \mathcal{D}_{2D}(\mathbf{W}_{\text{aux}, i})$.
            \item Add low-frequency components of $\mathbf{\Phi}_{\text{aux}, i}$ (as defined by the mask $\mathbf{M}_r$) to $\mathcal{G}_2$.
        \end{enumerate}
        \item Output $\mathcal{G}_2$ as learngenes.
    \end{enumerate}
\end{enumerate}

\paragraph{Algorithm 2: MLP Initialization of Variable-Sized Models}
\label{app:mlp_a2}
Applying Algorithm~\ref{app:mlp_a2} for Initialization to MLPs uses the 2D spectral learngenes to reconstruct weight matrices for a target MLP model with potentially different layer dimensions $D'_{\text{in}}, D'_{\text{out}}$.

\textbf{Input:} Acquired MLP Learngenes $\mathcal{G} = \{ \hat{\mathbf{\Phi}}_i \}$, target MLP layer dimensions $D'_{\text{in}}, D'_{\text{out}}$ for each corresponding layer.

\textbf{Output:} Initialized weight matrices for the variable-sized target MLP model $\mathbf{\Omega}' = \{ \mathbf{W}'^{(l)} \}_{l=1}^{L'}$.

\begin{enumerate}[label=\arabic*.]
    \item \textbf{Spectral Transformation:} For each set of learngenes $\hat{\mathbf{\Phi}}_i$ corresponding to a layer and its target dimensions $D'_{\text{in}}, D'_{\text{out}}$:
    \begin{enumerate}[label=\alph*.]
        \item Apply padding or truncation to $\hat{\mathbf{\Phi}}_i$ along its two dimensions to obtain $\mathbf{\Phi}'_i \in \mathbb{R}^{D'_{\text{in}} \times D'_{\text{out}}}$. The cases for padding/truncation are analogous to those described in Alg.~\ref{app:vit_a2} 2 (ViT), but applied to the two dimensions ($D_{\text{in}}, D_{\text{out}}$) of the spectral coefficients.
    \end{enumerate}
    \item \textbf{Weight Reconstruction:}
    \begin{enumerate}[label=\alph*.]
        \item Apply 2D-IDCT: $\mathbf{W}'_i = \mathcal{D}^{-1}_{2D}(\mathbf{\Phi}'_i) \in \mathbb{R}^{D'_{\text{in}} \times D'_{\text{out}}}$.
        \item Add the reconstructed weight matrix $\mathbf{W}'_i$ to the set $\mathbf{\Omega}'$.
        \item Output $\mathbf{\Omega}'$ as the initialized target MLP model weights.
    \end{enumerate}
\end{enumerate}

\subsection{Application to Convolutional Neural Networks (CNNs)}
\label{app:front_cnn}

Applying the FRONT framework to Convolutional Neural Networks (CNNs) follows the same principles, but the spectral transformation must be adapted to the dimensionality of CNN weight tensors. Convolutional layers typically have 4D weight tensors $\mathbf{W} \in \mathbb{R}^{C_{\text{out}} \times C_{\text{in}} \times K_h \times K_w}$, where $C_{\text{out}}$ and $C_{\text{in}}$ are the output and input channels, and $K_h, K_w$ are the kernel dimensions. Biases are usually 1D and handled separately. It is important to note that the four-dimensional tensor of a CNN can be reshaped into a two-dimensional tensor, and subsequently, multiple tensors can be combined into a three-dimensional tensor, and so forth. In order to demonstrate the flexibility of FRONT, the following presentation will outline the method for direct processing of the four-dimensional tensor.

\paragraph{Algorithm 1: CNN Learngene Acquisition}
\label{app:cnn_a1}
Applying Algorithm~\ref{app:cnn_a1} for Learngene Acquisition to CNNs utilizes the 4D Discrete Cosine Transform ($\mathcal{D}_{4D}$) and its inverse ($\mathcal{D}^{-1}_{4D}$).

\textbf{Requirements:}
\begin{itemize}
    \item Pre-trained/Auxiliary CNN model weight tensors $\mathbf{\Omega}_{\text{pre/aux}} = \{ \mathbf{W}^{(l)} \}_{l=1}^L $, frequency ratio $r$. Here, each $\mathbf{W}^{(l)} \in \mathbb{R}^{C^{(l)}_{\text{out}} \times C^{(l)}_{\text{in}} \times K^{(l)}_h \times K^{(l)}_w}$.
    \item For Method 2: Training dataset $\{ (x^{(i)}, y^{(i)}) \}_{i=1}^m$, number of epochs $N_{ep}$, batch size $B$, learning rate $\alpha$, regularization weight $\lambda$. (Decay rates $\gamma_d$ if applicable).
\end{itemize}

\textbf{Output}: Learngenes (Compact 4D spectral representation).

\begin{enumerate} % Main Level (Method 1 and 2)
    \item \textbf{Method 1: Direct Truncation of pre-trained Weights' Spectra}:
    \begin{enumerate} % Level 2 (Steps in Method 1)
        \item Initialize an empty set for learngene $\mathcal{G}_1 = \emptyset$.
        \item For each weight tensor $\mathbf{W}_i \in \mathbf{\Omega}_{\text{pre}}$:
        \begin{enumerate} % Level 3 (Steps inside the loop)
            \item Apply 4D-DCT: $\mathbf{\Phi}_i = \mathcal{D}_{4D}(\mathbf{W}_i)$.
            \item Construct binary truncation mask: Adapt Eq. (~\ref{eq:truncation}) for 4D, using indices $(c_{\text{out}}, c_{\text{in}}, k_h, k_w)$ and dimensions $(C_{\text{out},i}, C_{\text{in},i}, K_{h,i}, K_{w,i})$ of $\mathbf{W}_i$. The mask $\mathbf{M}_r[c_{\text{out}}, c_{\text{in}}, k_h, k_w]$ is 1 if $c_{\text{out}} \leq \lfloor r C_{\text{out},i} \rfloor$, $c_{\text{in}} \leq \lfloor r C_{\text{in},i} \rfloor$, $k_h \leq \lfloor r K_{h,i} \rfloor$, and $k_w \leq \lfloor r K_{w,i} \rfloor$, and 0 otherwise.
            \item Apply truncation: $\hat{\mathbf{\Phi}}_i = \mathbf{M}_r \odot \mathbf{\Phi}_i$.
            \item Add the non-zero elements of $\hat{\mathbf{\Phi}}_i$ to $\mathcal{G}_1$.
        \end{enumerate}
        \item Output $\mathcal{G}_1$ as Learngenes.
    \end{enumerate}

    \item \textbf{Method 2: Distillation-Aware Training with a high-frequency attenuation regularization}:
    \begin{enumerate} % Level 2 (Steps in Method 2)
        \item Initialize Auxiliary CNN Model with weights $\mathbf{\Omega}_{\text{aux}}$ (as in Alg.~\ref{app:vit_a1}, Method 2 for ViT).
        \item Initialize training optimizer for $\mathbf{\Omega}_{\text{aux}}$.
        \item For $ep = 1$ to $N_{ep}$: Steps for batch processing, logit loss $\mathcal{L}_{\text{logit}}$, total loss $\mathcal{L}_{\text{total}}$, and weight updates are analogous to those in Alg.~\ref{app:vit_a1}, Method 2 for ViT.
        \begin{enumerate} % Level 3 (Steps inside the outer loop)
             \item For each batch $\{ (x_j, y_j) \}_{j=1}^B$:
             \begin{enumerate} % Level 4 (Steps inside the batch loop)
                \item Get logits $z_{\text{pre},j}, z_{\text{aux},j}$ and calculate $\mathcal{L}_{\text{logit}}$ (as in Alg.~\ref{app:vit_a1}, Method 2 for ViT).
                \item Calculate $\mathcal{L}_{\text{reg}}$:
                \begin{itemize}
                    \item Set $\mathcal{L}_{\text{reg}} = 0$.
                    \item For each weight tensor $\mathbf{W}_{\text{aux}, i} \in \mathbf{\Omega}_{\text{aux}}$:
                    \begin{itemize}
                        \item Apply 4D-DCT: $\mathbf{\Phi}_{\text{aux}, i} = \mathcal{D}_{4D}(\mathbf{W}_{\text{aux}, i})$.
                        \item Construct adaptive mask $\mathbf{M}_r[c_{\text{out}}, c_{\text{in}}, k_h, k_w]$ using adapted Eq. (~\ref{m}) for 4D dimensions.
                        \item Calculate term: Adapt Eq. (~\ref{reg}) for 4D dimensions:
                        \begin{align*}
                            Term_i = \frac{1}{\|\mathbf{M}_r \odot \mathbf{\Phi}_{\text{aux}, i}\|_0} \sum_{c_{\text{out}}, c_{\text{in}}, k_h, k_w} (\mathbf{M}_r[c_{\text{out}}, c_{\text{in}}, k_h, k_w] \cdot \mathbf{\Phi}_{\text{aux}, i}[c_{\text{out}}, c_{\text{in}}, k_h, k_w])^2
                        \end{align*}
                        \item $\mathcal{L}_{\text{reg}} = \mathcal{L}_{\text{reg}} + Term_i$.
                    \end{itemize}
                \end{itemize}
                \item Calculate Total Loss $\mathcal{L}_{\text{total}} = (1-\lambda)\mathcal{L}_{\text{logit}} + \lambda\mathcal{L}_{\text{reg}}$.
                \item Perform backward pass and update weights $\mathbf{\Omega}_{\text{aux}}$ (as in Alg.~\ref{app:vit_a1}, Method 2 for ViT).
             \end{enumerate}
        \end{enumerate}
        \item Initialize $\mathcal{G}_2 = \emptyset$.
        \item For each trained weight tensor $\mathbf{W}_{\text{aux}, i} \in \mathbf{\Omega}_{\text{aux}}$:
        \begin{enumerate} % Level 3 (Steps inside the loop)
            \item Apply 4D-DCT: $\mathbf{\Phi}_{\text{aux}, i} = \mathcal{D}_{4D}(\mathbf{W}_{\text{aux}, i})$.
            \item Add low-frequency components of $\mathbf{\Phi}_{\text{aux}, i}$ (as defined by the mask $\mathbf{M}_r$) to $\mathcal{G}_2$.
        \end{enumerate}
        \item Output $\mathcal{G}_2$ as learngenes.
    \end{enumerate}
\end{enumerate}

\paragraph{Algorithm 2: CNN Initialization of Variable-Sized Models}
\label{app:cnn_a2}
Applying Algorithm~\ref{app:cnn_a2} for Initialization to CNNs uses the 4D spectral learngenes to reconstruct weight tensors for a target CNN model with potentially different layer dimensions $C'_{\text{out}}, C'_{\text{in}}, K'_h, K'_w$.

\textbf{Input:} Acquired CNN Learngenes $\mathcal{G} = \{ \hat{\mathbf{\Phi}}_i \}$, target CNN layer dimensions $C'_{\text{out}}, C'_{\text{in}}, K'_h, K'_w$ for each corresponding layer.

\textbf{Output:} Initialized weight tensors for the variable-sized target CNN model $\mathbf{\Omega}' = \{ \mathbf{W}'^{(l)} \}_{l=1}^{L'}$.

\begin{enumerate}[label=\arabic*.]
    \item \textbf{Spectral Transformation:} For each set of learngenes $\hat{\mathbf{\Phi}}_i$ corresponding to a layer and its target dimensions $C'_{\text{out}}, C'_{\text{in}}, K'_h, K'_w$:
    \begin{enumerate}[label=\alph*.]
        \item Apply padding or truncation to $\hat{\mathbf{\Phi}}_i$ along its four dimensions to obtain $\mathbf{\Phi}'_i \in \mathbb{R}^{C'_{\text{out}} \times C'_{\text{in}} \times K'_h \times K'_w}$. The cases for padding/truncation are analogous to those described in Alg.~\ref{app:vit_a2} (ViT), but applied to the four dimensions ($C_{\text{out}}, C_{\text{in}}, K_h, K_w$) of the spectral coefficients.
    \end{enumerate}
    \item \textbf{Weight Reconstruction:}
    \begin{enumerate}[label=\alph*.]
        \item Apply 4D-IDCT: $\mathbf{W}'_i = \mathcal{D}^{-1}_{4D}(\mathbf{\Phi}'_i) \in \mathbb{R}^{C'_{\text{out}} \times C'_{\text{in}} \times K'_h \times K'_w}$.
        \item Add the reconstructed weight tensor $\mathbf{W}'_i$ to the set $\mathbf{\Omega}'$.
        \item Output $\mathbf{\Omega}'$ as the initialized target CNN model weights.
    \end{enumerate}
\end{enumerate}

\section{Training Details}
\label{app:all_training_details}

\subsection{Supplemented Experimental Setup}
We systematically evaluate the effectiveness of our method across different architectures and tasks, organizing our analysis by modality. For vision models, we first assess scalability by varying the depth and width of DeiT in Section~\ref{sec:init_ability}, including a thorough 300-epoch analysis of direct initialization quality. We then examine cross-dataset and cross-task generalization. For language models, we conduct analogous scaling experiments and evaluate downstream performance on the GLUE benchmark for models like BERT in Section~\ref{sec:llm}. Finally, Section~\ref{sec:ablation} presents ablation and analyses to understand the underlying mechanisms that make our method effective.

\textbf{Datasets.} For vision tasks, all source models, whether used for direct extraction or refinement, are pre-trained on ImageNet-1K. This dataset is also used for our vision scaling experiments. We assess cross-dataset transferability on a diverse suite of seven downstream classification datasets. Furthermore, we evaluate generalization to other vision domains, using six datasets for object detection and four for image segmentation. For language tasks, we conduct training on standard large-scale corpora: the English Wikipedia corpus is used for BERT~\citep{devlin-etal-2019-bert} and RoBERTa~\citep{liu2019roberta}, while the concatenation of English Wikipedia and the Toronto Book Corpus is used for GPT-2~\citep{liu2019roberta}. The downstream performance of our initialized language models is subsequently evaluated on the GLUE benchmark~\citep{wang2018glue}.

\textbf{Architectural Details.} Our evaluation spans a wide range of architectures in both vision and language domains. For vision models, we utilize publicly available DeiT-Ti/S/B models pre-trained on ImageNet-1K as the source for direct extraction. For our refinement, we train three compact 8-layer auxiliary models (with 3, 6, or 12 heads) from scratch for 150 epochs on ImageNet-1K. We also explore a more efficient refinement variant by briefly fine-tuning the pre-trained DeiT models. To assess scalability, we initialize DeiT variants by varying their depth (4–12 layers) and width (6–24 heads; 384–1536 dimensions). For generalization experiments, we initialize ResNet-50/152 from a ResNet-101 source. For language models, we demonstrate cross-scale transfer by initializing a 6-layer, 384-dimension model (e.g., BERT-S) using learngenes extracted from its corresponding 12-layer, 768-dimension Base counterpart (BERT-B). This procedure is applied across three foundational architectures: BERT, RoBERTa, and GPT-2. For more experimental details, including visual and language models and datasets, please refer to the below appendix.

\textbf{Baseline.} To provide a comprehensive evaluation for vision tasks, we benchmark our framework's flexible learngene sourcing strategies against two main categories of initialization methods:\textbf{ (1) Direct Initialization.}  Methods here rely on prior knowledge or direct parameter transfer applied to existing pre-trained models on ImageNet-1K without additional training. This includes He-Init~\citep{Chen_Xie_He_2021}, Mimetic-Init~\citep{trockman2023mimetic}, Wt Select~\citep{xu2023initializing}, Heur-LG~\citep{wang2022learngene}, Cluster-LG~\citep{wang2024clusterlearngene}, LiGO~\citep{wang2023learning}, and \textbf{FRONT{}}. \textbf{(2) Methods with Extra Training.} 
This category encompasses methods that require a dedicated optimization phase to generate transferable knowledge. Methods like GHN-3~\citep{knyazev2023can}, Share-Init~\citep{lan2019albert}, Auto-LG~\citep{wang2023learngene}, TELG~\citep{xia2024transformer}, and WAVE~\citep{feng2024wave}, execute a dedicated, computationally intensive process to generate or discover transferable parameters, a necessity as their frameworks typically do not support the direct fine-tuning of the pre-trained models. To ensure the most direct and rigorous comparison against these high-effort methods, our main experiments also adopt the from-scratch paradigm with \textbf{FRONT{\scalebox{1}{\text{+}}}}. We use the same pretraining data and experimental settings for all the methods for a fair comparison. Notably, our refinement framework also supports a highly efficient fine-tuning strategy (\textbf{FRONT{\scalebox{0.8}{\text{++}}}}). Our preliminary results indicate that this approach can surpass the from-scratch version in performance with significantly fewer training cost, detailed in Section~\ref{sec:refine}. For language tasks, we established two baselines: training the Small model from scratch and employing knowledge distillation~\citep{hinton2015distilling}, where the respective Base model served as the teacher.

\textbf{Evaluation Metrics.} The main metric is Top-1 accuracy, measuring initialization effectiveness. Supplementary metrics are convergence efficiency (epochs) and parameter transfer efficiency.

\subsection{Vision Tasks}
\label{app:training_details}
\subsubsection{Hyper-parameters}
\label{app:hyper}
Table~\ref{tab:hyper_main} and Table~\ref{tab:hyper_down} detail the hyperparameters used for FRONT+, which concentrates knowledge into low frequencies, and for training models initialized with this low-frequency knowledge on various datasets. These hyperparameters include batch sizes, warm-up epochs, training epochs, and other relevant settings.

\renewcommand{\arraystretch}{0.9}
\begin{table}
    \centering
    \caption{Hyper-parameters for FRONT's Retraining on ImageNet-1K.}
    % \vspace{-0.1in}
    \setlength{\tabcolsep}{18 mm}
        \begin{tabular}{@{}lr@{}}
        \toprule[1.3pt]
        \textbf{Training Settings} & \textbf{Configuration} \\
        \midrule[1.1pt]
        optimizer & AdamW\\
        base learning rate & Ti: 5e-4 $\mid$ S: 2.5e-4 $\mid$ B: 1.25e-4\\
        warmup learning rate & 1e-6\\
        weight decay & 0.05\\
        optimizer momentum & 0.9\\
        batch size & Ti: 512 $\mid$ S: 256 $\mid$ B: 128\\
        training epochs & 150\\
        learning rate schedule & cosine decay\\
        warmup epochs & 5\\
        color jitter & 0.4 \\
        auto augment & rand-m9-mstd0.5-inc1\\
        mixup & 0.8\\
        cutmix & 1.0\\
        label smoothing & 0.1\\
        drop path & 0.1\\
        \bottomrule[1.3pt]
        \end{tabular}
    \label{tab:hyper_main}
    \vspace{-0.05in}
\end{table}

\begin{table*}
    \centering
    \caption{Hyperparameters for neural networks trained on downstream datasets.}
    \vspace{-0.1in}
    \resizebox{\textwidth}{!}{
        \begin{tabular}{@{}lccccccccccccc@{}}
        \toprule[1.3pt]
        \textbf{Dataset} & \textbf{\shortstack{Batch\\ Size}} & \textbf{\shortstack{Epochs\\DeiT/ResNet}} & \textbf{\shortstack{Learning\\ Rate}} & \textbf{\shortstack{Drop\\ Last}} & \textbf{\shortstack{Warmup\\ Epochs}} & \textbf{\shortstack{Droppath\\ Rate}} & \textbf{\shortstack{Color\\ Jitter}} & \textbf{\shortstack{Auto\\ Augment}} & \textbf{\shortstack{Random\\ Erase}} & \textbf{Mixup} & \textbf{Cutmix} & \textbf{Scheduler} & \textbf{Optimizer}\\
        \midrule[1.1pt]
        \textbf{Oxford Flowers} & 512 & 300/100 & 3e-4 & False & 0 & 0 & 0.4 & \multirow{7}{*}{\rotatebox{90}{\fontsize{7.5}{12}\selectfont rand-m9-mstd0.5-inc1}} & 0.25 & 0 & 0 & cosine & AdamW \\
        \textbf{CUB-200-2011} & 512 & 300/100 & 3e-4 & False & 0 & 0.1 & 0 & & 0.25 & 0 & 0 & cosine & AdamW \\
        \textbf{Stanford Cars} & 512 & 300/100 & 3e-4 & False & 0 & 0.1 & 0 & & 0.25 & 0 & 0 & cosine & AdamW \\
        \textbf{CIFAR10} & 512 & 300/100 & 5e-4 & True & 0 & 0.1 & 0.4 & & 0.25 & 0 & 0 & cosine & AdamW\\
        \textbf{CIFAR100} & 512 & 300/100 & 5e-4 & True & 0 & 0.1 & 0.4 & & 0.25 & 0 & 0 & cosine & AdamW\\
        \textbf{Food101} & 512 & 300/100 & 5e-4 & True & 0 & 0.1 & 0.4 & & 0.25 & 0 & 0 & cosine & AdamW\\
        \textbf{iNat-2019} & 512 & 100/50 & 5e-4 & True & 0 & 0.1 & 0.4 & & 0.25 & 0 & 0 & cosine & AdamW\\
        \bottomrule[1.3pt]
        \end{tabular}
        }
    \label{tab:hyper_down}
\end{table*}

\subsubsection{Details of FRONT's Direct Setting}
One-time extraction with primary hyperparameter being extracted parameter quantity. 

\textbf{Depth Expansion}: For tiny models, search space is \{1.7, 2.2, 2.7 \}M parameters, with 2.2M selected as optimal. For small models, search space is \{8.1, 11.4\}M parameters, with 8.1M selected as optimal. For base models, search space is \{32.4, 46.3\}M parameters, with 32.4M selected as optimal. 

\textbf{Width Expansion}:We uniformly use a parameter size of 8.9M for transmission.

\textbf{Downstream Tasks}: In the settings of Table~\ref{tab:downstream}, for the tiny model, we initialize the 6-layer tiny model using a publicly available 12-layer pre-trained tiny model, transferring 2.2M parameters. Similarly, for the small model, the 6-layer small model is initialized with a 12-layer publicly available pre-trained small model, transferring 8.1M parameters. For ResNets, we initialize the ResNet50/152 models with weights from a pre-trained ResNet101 model.

\subsubsection{Details of FRONT+'s Refinement}
Our Refinement strategy is implemented via two distinct approaches:

\textbf{Training from Scratch (FRONT{\scalebox{1}{\text{+}}}).} For our main experiments requiring a from-scratch model, we train compact auxiliary networks for 150 epochs on ImageNet-1K. To enhance their representational quality, we employ knowledge distillation, using either a LeViT-384 or a RegNetY-16GF model as the teacher. The learngenes are then directly extracted from these fully trained auxiliary models. A detailed ablation study on the choice of teacher models can be found in Appendix~\ref{app:larger_anc}.

Hyperparameters include regularization coefficient in Eq.~(~\ref{eq:total_loss}), transferred parameter quantity and learning rate. We maintain consistent parameter quantities for fair comparison: 0.8M for tiny models, 3.2M for small models, 13.0M for base models. Regularization coefficient searched over \{0.5, 0.2, 0.1, 0.05, 0.002, 0.001, 0.0001\}, with 0.002 selected based on convergence and regularization effectiveness. We select the learning rate for search in \{1e-3, 5e-4\}. During the training of the auxiliary model, we use $l = 2$ and set 
$m, n = 0.25 \times d_\text{in}, d_\text{out}$ 
in the mask matrix, where $d_\text{in}$ and $d_\text{out}$ denote the input and output dimension, respectively.

\textbf{Fine-tuning a Pre-trained Model (FRONT{\scalebox{0.8}{\text{++}}}).} To provide a more computationally efficient alternative, this approach initializes with a publicly available pre-trained model and then briefly fine-tunes it with our frequency regularization term. As detailed in our ablation analysis (Section~\ref{sec:refine}), we conduct a grid search to determine the optimal configuration. Specifically, a DeiT-Tiny-L8 source model is fine-tuned on ImageNet-1K across a range of epochs \{10, 20, 50, 100\} and regularization strengths ($\lambda$) \{1e-5, 5e-4, 1e-3, 2e-3, 5e-3\}. The quality of the learngenes (0.8M parameters) produced by each configuration is then assessed by their effectiveness in initializing a deeper DeiT-Tiny-L10 target model.

\subsubsection{Details of Downstream Datasets}
\label{app:dataset}
Additional datasets include Oxford Flowers~\citep{nilsback2008automated}, CUB-200-2011~\citep{wah2011caltech}, Stanford Cars~\citep{gebru2017fine}, CIFAR-10, CIFAR-100~\citep{krizhevsky09}, Food-101~\citep{bossard2014food}, and iNaturalist-2019~\citep{tan2019herbarium}.
Table~\ref{tab:datasets} presents the details of seven downstream datasets, which are sorted by the size of datasets.

\begin{table}
    \centering
    \caption{Characteristics of downstream datasets.}
    % \vspace{-0.1in}
    \resizebox{0.68\textwidth}{!}{
        \begin{tabular}{@{}lcccc@{}}
        \toprule[1.3pt]
        \textbf{Dataset} & \textbf{Classes} & \textbf{Total} & \textbf{Training} & \textbf{Testing} \\
        \cmidrule[1.1pt]{1-5}
        \textbf{Oxford Flowers}~\citep{nilsback2008automated} & 102 & 8,189 & 2,040  & 6,149\\
        \textbf{CUB-200-2011}~\citep{wah2011caltech} & 200 & 11,788 & 5,994 & 5,794 \\
        \textbf{Stanford Cars}~\citep{gebru2017fine} & 196 & 16,185 & 8,144 & 8,041\\
        \textbf{CIFAR10}~\citep{krizhevsky09} & 10 & 60,000 & 50,000 & 10,000 \\
        \textbf{CIFAR100}~\citep{krizhevsky09} & 100 & 60,000 & 50,000 & 10,000 \\
        \textbf{Food101}~\citep{bossard2014food} & 101 & 101,000 & 75,750 & 25,250\\
        \textbf{iNat-2019}~\citep{tan2019herbarium} & 1010 & 268,243 & 265,213 & 3,030\\
        \bottomrule[1.3pt]
        \end{tabular}
        }
    \label{tab:datasets}
\end{table}

% Following the setting~\citep{nie2024cross}, it is clear that FRONT effectively transfers knowledge acquired in the source domain to the four target domains compared to from scratch. On DeepGlobe~\citep{demir2018deepglobe}, FRONT achieves comparable accuracy when segmenting various categories from satellite imagery. Furthermore, FRONT produces precise segmentations for medical screening data from ISIC~\citep{codella2019skin} and Chest X-Ray~\citep{tschandl2018ham10000}. Regarding FSS-1000~\citep{li2020fss}, FRONT delivers strong performance when predicting a wide range of target categories, such as logos, food items, and other objects.

% Following the setting~\citep{fu2024cross}, COCO~\citep{lin2014microsoft} is a widely adopted dataset for object detection, offering a broad range of categories including humans, animals, vehicles, and common items; it is utilized as the source domain dataset. The remaining six datasets—ArTaxOr~\citep{drange2019arthropod}, Clipart1k~\citep{inoue2018cross}, DIOR~\citep{li2020object}, DeepFish ~\citep{saleh2020realistic}, NEUDET~\citep{song2013noise}, and UODD~\citep{jiang2021underwater}—are employed as target domain datasets.

\subsection{Language Tasks}
To evaluate the effectiveness and adaptability of our proposed method, we conducted comparative experiments on language models. We applied FRONT to initialize a 6-layer, 384-dimensional  ``Small'' model (e.g., BERT-S) from its corresponding pre-trained 12-layer, 768-dimensional ``Base'' (e.g., BERT-B) model. This procedure was performed for three distinct architectures: BERT, RoBERTa, and GPT-2. For comparison, we established two baselines: training the Small model from scratch and employing knowledge distillation~\citep{hinton2015distilling}, where the respective Base model served as the teacher. We use the English Wikipedia corpus for training BERT~\citep{devlin-etal-2019-bert} and RoBERTa~\citep{liu2019roberta}. We use the concatenation of English Wikipedia and Toronto Book Corpus for training GPT2~\citep{radford2019language}. We remove the next sentence prediction task (~\citep{liu2019roberta}) and use a fixed sequence length of 128 for pretraining both BERT and RoBERTa. For BERT, we use a batch size of 256 and a learning rate of 2e-4, while we use a batch size of 1024 and a learning rate of 8e-4 for training RoBERTa models. For GPT2, we use a batch size of 512, a fixed sequence length of 1024 and a learning rate of 1e-3.

Subsequently, we evaluated the downstream performance of the initialized BERT-S model on the GLUE benchmark. For fine-tuning, we used the Adam optimizer with a learning rate of 2e-5, a batch size of 32, and searched for the optimal number of epochs from \{3, 6, 10\}. The FRONT-initialized models were directly fine-tuned on each task and knowledge distillation requires the corresponding base models to be used as teacher models throughout the entire training process.

\section{Additional Results}
\label{app:additional_results}

% ----------------------------------------------------------------------
% --- 请将此部分放入您论文的附录 (Appendix) ---
% ----------------------------------------------------------------------

\subsection{Ablation Study on Transform Basis Choice}
\label{sec:appendix_dct_ablation}

% 根据您的要求，所有为 rebuttal 期间添加的新内容都标为红色

Regarding our choice of the DCT over other transforms, we provide this section with a concise theoretical justification followed by a comprehensive empirical ablation study.

\begin{table}[ht]
\centering
\setlength{\tabcolsep}{6 mm}
\caption{Comparison of different basis transforms for knowledge transfer on ImageNet-1K. Our DCT-based method (FRONT) achieves the \textbf{highest} top-1 accuracy (\%) in both depth and width transfer tasks. This validates our design choice.}
\label{tab:ablation_transforms}
\resizebox{\textwidth}{!}{
\begin{tabular}{l|cccccc}
\toprule
Methods & Scratch & \cellcolor{blue!12}\textbf{FRONT} & DFT & DWT & PCA & SVD \\
\midrule
DeiT-Ti\_L12 $\rightarrow$ DeiT-Ti\_L6 & 40.6 & \cellcolor{blue!12}\textbf{63.3} & 60.8 & 62.0 & 57.9 & 53.0 \\
DeiT-S\_L6 $\rightarrow$ DeiT-Ti\_L6 & 40.6 & \cellcolor{blue!12}\textbf{56.9} & 47.9 & 55.4 & 50.2 & 43.6 \\
\bottomrule
\end{tabular}
}
\end{table}

Our primary motivation for using DCT is that ``low-frequency" in the DCT domain in the signal processing has a universal, and semantically meaningful interpretation: \textbf{smoothness}. The cosine basis functions are fixed \textit{a priori}. This allows us to posit a testable hypothesis: that the generalizable, task-agnostic knowledge in weights is encoded in these universal patterns. Furthermore, DCT is a \textbf{real-to-real} transform, which is perfectly compatible with real-valued neural network weights.

Our theoretical analysis of the alternatives is as follows:
\begin{itemize}
    \item \textbf{DFT (Fourier):} DFT's primary drawback is that it is a \textbf{complex-valued} transform. It maps real-valued weights to complex coefficients (magnitude and phase), forcing a non-trivial choice: either discard the phase (information loss) or double the parameter count. DCT avoids this ambiguity.
    \item \textbf{DWT (Wavelets):} DWT's strength is its excellent \textbf{spatial localization} (analyzing \textit{where} a frequency occurs). For our task, this localization is unnecessary. Our hypothesis is about the \textbf{global spectral properties} of the entire weight tensor, not the local position of specific patterns.
    \item \textbf{PCA/SVD (Intra-Model Compression):} This is the key distinction. PCA and SVD are powerful tools for \textit{intra-model compression}, not \textit{cross-model initialization}. It is ill-defined how to use the components (PCA) or singular vectors (SVD) from a source model's weight matrix (e.g., $1024 \times 512$) to initialize a target model's matrix (e.g., $512 \times 256$) without destroying the structural knowledge we aim to transfer.
\end{itemize}
Therefore, DCT is uniquely suited as it is a \textbf{structure-preserving transform}, not a dimensionality reduction or decomposition technique.

\begin{wraptable}{r}{0.32\textwidth}
    \centering
    \setlength{\tabcolsep} {2.8 mm}
    \vspace{-0.17in}
    \caption{
    % Knowledge transfer across diverse frequency components. 
    Low-frequency knowledge surpasses others.}
    \vspace{-0.1in}
    \resizebox{\linewidth}{!}{
    \begin{tabular}{@{}l|ccccc@{}}
      \toprule
      \multirow{2}{*}{\textbf{\shortstack{Select\\Freq.}}} & \multicolumn{5}{c}{\textbf{DeiT-Ti}} \\
      \cmidrule{2-6}
      & \textbf{L4} & \textbf{L6} & \textbf{L8} & \textbf{L10} & \textbf{L12} \\
      \midrule
      He-Init & 34.7 & 40.6 & 43.7 & 46.8 & 48.3\\
      \midrule
      High & 32.6 & 36.0 & 35.8 & 38.9 & 39.0 \\
      Mid & 34.0 & 34.6 & 37.0 & 38.1 & 38.2 \\
      \cellcolor{blue!12}{Low} & \cellcolor{blue!12}{\textbf{48.6}} & \cellcolor{blue!12}{\textbf{53.7}} & \cellcolor{blue!12}{\textbf{56.5}} & \cellcolor{blue!12}{\textbf{58.6}} & \cellcolor{blue!12}{\textbf{59.6}} \\
      \bottomrule
    \end{tabular}
    }
    \label{tab:freq_comparison}
    \vspace{-0.16in}
\end{wraptable}

While the theoretical arguments above guided our design, we agree that empirical validation is crucial. We conducted a comprehensive ablation study comparing our DCT-based method (FRONT) against baselines using DFT, DWT, PCA, and SVD. We performed experiments for both depth and width transformations.
\textbf{Depth Transfer:} Using a 12-layer DeiT-Ti to initialize a 6-layer DeiT-Ti.
\textbf{Width Transfer:} Using a 6-layer DeiT-S to initialize a 6-layer DeiT-Ti.
All other training parameters are kept consistent with the settings in the main paper's Table~\ref{tab:lenth}. The results are shown in Table \ref{tab:ablation_transforms}.

% \begin{wraptable}{r}{0.38\textwidth}
% \centering
% \setlength{\tabcolsep}{0.8 mm}
% \vspace{-0.17in}
% \caption{
% % Performance comparison between random initialization and FRONT on various dataset.
% Performance comparison on ResNet50 and ResNet152.
% }
% % \vspace{-0.12in}
% \resizebox{\linewidth}{!}{ % 这里让表格宽自动等于wraptable的宽
% \begin{tabular}{lccl|ccl}
% \toprule
%     & \multicolumn{3}{c|}{ResNet50} & \multicolumn{3}{c}{ResNet152} \\
%     \cmidrule{2-7}
%     & Rand. & \multicolumn{2}{c|}{FRONT{}} & Rand. & \multicolumn{2}{c}{FRONT{}} \\
%     \midrule
%     Image. & 74.1 & 76.1 & \textcolor{mygreen}{$\uparrow$ 2.0} & 76.8 & 77.2 & \textcolor{mygreen}{$\uparrow$ 0.4} \\
%     Flow. & 51.3 & 91.3 & \textcolor{mygreen}{$\uparrow$ 40.0} & 28.0 & 87.7 & \textcolor{mygreen}{$\uparrow$ 59.7} \\
%     CUB & 40.7 & 68.9 & \textcolor{mygreen}{$\uparrow$ 28.2} & 22.8 & 68.2 & \textcolor{mygreen}{$\uparrow$ 45.4} \\
%     Cars & 48.6 & 88.2 & \textcolor{mygreen}{$\uparrow$ 39.6} & 29.3 & 87.4 & \textcolor{mygreen}{$\uparrow$ 58.1} \\
%     C10 & 94.8 & 95.6 & \textcolor{mygreen}{$\uparrow$ 0.8} & 95.6 & 95.7 & \textcolor{mygreen}{$\uparrow$ 0.1} \\
%     C100 & 76.8 & 78.3 & \textcolor{mygreen}{$\uparrow$ 1.5} & 77.7 & 79.5 & \textcolor{mygreen}{$\uparrow$ 1.8} \\
%     Food & 82.1 & 85.8 & \textcolor{mygreen}{$\uparrow$ 3.7} & 83.2 & 85.9 & \textcolor{mygreen}{$\uparrow$ 2.7} \\
%     iNat & 61.8 & 69.0 & \textcolor{mygreen}{$\uparrow$ 7.2} & 62.7 & 69.1 & \textcolor{mygreen}{$\uparrow$ 6.4} \\
% \bottomrule
% \end{tabular}
% }
% \label{tab:performance_comparison}
% \vspace{-0.12in}
% \end{wraptable}

The empirical results in Table \ref{tab:ablation_transforms} strongly validate our design choices. Our DCT-based method, FRONT, clearly outperforms all other baselines, leveraging its energy compaction property. As hypothesized, the frequency-domain methods (DCT, DWT, DFT) are significantly superior to intra-model compression tools (PCA, SVD) for knowledge transfer, confirming our theoretical concerns about the inapplicability of the latter. Notably, while depth transfer (63.3\%) is highly effective, the performance drop in width transfer (56.9\%)—despite comparable source model performance—highlights the significant challenge posed by transformations that disrupt dimensional structure. Finally, the strong showing of DWT (achieving the second-best results) suggests it is a promising avenue for future research in this area.

\subsection{Generalization of Low-frequency Components}

To systematically evaluate the contribution of different frequency components to knowledge transfer, we initialize DeiT-Ti(192 hidden dimensions) using frequency-specific components extracted from each layer of DeiT-S(384 hidden dimensions). Given \( x \in \{4, 6, 8, 10, 12\} \), representing the number of layers considered, and for each layer \( i \) where \( 0 \leq i < x \), we partition the weight matrix \( W\text{-}S_i \) into low-frequency (\( W\text{-}S_i[:192, :192] \)), mid-frequency (\( W\text{-}S_i[96:288, 96:288] \)), and high-frequency (\( W\text{-}S_i[192:, 192:] \)) components. 
As shown in Table~\ref{tab:freq_comparison}, models initialized with low-frequency components consistently achieve the highest performance, highlighting the superior transferability.

\subsection{Integration of Knowledge from Larger pre-trained models}
\label{app:larger_anc}

\begin{table}[ht]
    \centering
    \caption{Additional results on different teacher models.}
    \resizebox{0.78\textwidth}{!}{
        \begin{tabular}{@{}lclll|lll@{}}
        \toprule[1.3pt]
        & & \multicolumn{3}{c}{DeiT-Ti} & \multicolumn{3}{c}{DeiT-S}\\
        \cmidrule{3-8}
        & Ancestry & \multicolumn{1}{c}{$L_{\text{6}}$} & \multicolumn{1}{c}{$L_{\text{8}}$} & \multicolumn{1}{c}{$L_{\text{12}}$} & \multicolumn{1}{c}{$L_{\text{4}}$} & \multicolumn{1}{c}{$L_{\text{8}}$} & \multicolumn{1}{c}{$L_{\text{12}}$} \\
        \toprule[1.1pt]
        TLEG~\citep{xia2024transformer} & LeVit-384 (39.1M) & 60.5 & 62.9 & 65.4 & 70.5 & 72.1 & 73.8 \\
        WAVE~\citep{feng2024wave} & LeVit-384 (39.1M) & 63.2 & 65.4 & 67.3 & 72.7 & 74.1 & 75.3 \\
        \midrule
        \cellcolor{blue!12}{FRONT{\scalebox{1}{\text{+}}}} & \cellcolor{blue!12}{LeVit-384 (39.1M)} & \cellcolor{blue!12}{\textbf{63.4}} & \cellcolor{blue!12}{\textbf{65.6}} & \cellcolor{blue!12}{\textbf{67.5}} & \cellcolor{blue!12}{\textbf{72.9}} & \cellcolor{blue!12}{\textbf{74.2}} & \cellcolor{blue!12}{\textbf{75.4}} \\
        \cellcolor{red!10}{FRONT{\scalebox{1}{\text{+}}}} & \cellcolor{red!10}{RegNet-16GF (83.6M)} & \cellcolor{red!10}{63.3$^*$} & \cellcolor{red!10}{65.2$^*$} & \cellcolor{red!10}{\textbf{67.5}} & \cellcolor{red!10}{72.4$^*$} & \cellcolor{red!10}{\textbf{74.2}} & \cellcolor{red!10}{75.3$^*$} \\
        \bottomrule[1.3pt]
        \end{tabular}
        }
    \label{tab:anc}
\end{table}

Low-frequency knowledge can be integrated with relevant information extracted from pre-trained models by filtering out size-specific components that violate transformation constraints. This is achieved by zero-padding or truncation in DCT / IDCT process. Such a mechanism ensures the efficient transfer and sharing of size-agnostic knowledge across models of varying sizes.

\begin{wraptable}{r}{0.32\textwidth}
    \centering
    \setlength{\tabcolsep}{2.8 mm}
    % \vspace{-0.18in}
    \caption{Pre-trained model's scale. Pre-trained models who are more similar provide better initialization.}
    \vspace{-0.1in}
    \resizebox{0.32 \textwidth}{!}{
    \begin{tabular}{@{}l|ccccc@{}}
      \toprule
      \multirow{2}{*}{\textbf{\shortstack{PT.\\Model}}} & \multicolumn{5}{c}{\textbf{DeiT-Ti}} \\
      \cmidrule{2-6}
      & \textbf{L4} & \textbf{L6} & \textbf{L8} & \textbf{L10} & \textbf{L12} \\
      \midrule
      He-Init & 34.7 & 40.6 & 43.7 & 46.8 & 48.3\\
      \midrule
      % DeiT-Ti & \textbf{55.3} & \textbf{63.3} & \textbf{64.4} & \textbf{64.7} & \textbf{65.3}\\
      DeiT-B & 29.3 & 39.2 & 47.7 & 51.8 & 57.1 \\
      DeiT-S & 34.8 & 45.5 & 51.9 & 56.7 & 59.7 \\   
      \cellcolor{blue!12}{DeiT-Ti} & \cellcolor{blue!12}{\textbf{55.3}} & \cellcolor{blue!12}{\textbf{63.3}} & \cellcolor{blue!12}{\textbf{64.4}} & \cellcolor{blue!12}{\textbf{64.7}} & \cellcolor{blue!12}{\textbf{65.3}} \\
      \bottomrule
    \end{tabular}
    }
    \label{tab:model_comparison}
    \vspace{-0.1in}
\end{wraptable}

To evaluate the influence of pre-trained teacher models with varying architectures and scales, we incorporated a larger pre-trained model, RegNet-16GF (83.6M parameters)~\citep{radosavovic2020designing}, and compared FRONT{\scalebox{1}{\text{+}}} to WAVE and TLEG, both of which utilize LeVit-384 (39.1M parameters) as an alternative teacher model. Table~\ref{tab:anc} presents the results for DeiT-Ti and DeiT-S at each layer.

The results demonstrate that FRONT{\scalebox{1}{\text{+}}} consistently outperforms the other methods across all model sizes. Although the inclusion of a larger teacher model (RegNet-16GF) provides some improvement, the gains are relatively modest. This suggests that once pre-trained models are sufficiently trained and informative, the shared low-frequency knowledge becomes effective and robust to further increases in teacher model size. These findings highlight the effectiveness and stability of FRONT{\scalebox{1}{\text{+}}} in condensing and integrating low-frequency knowledge from various retrained pre-trained models.

\subsection{Effect of Source Pre-trained Model Scale for Learngenes}

% \begin{table}[ht]
%   \centering
%   \begin{minipage}{0.48\textwidth}
%     \centering

%     \caption{Ablation of frequency domain. Low-frequency knowledge outperforms others.}
%     \begin{tabular}{p{1.1cm}|ccccc}
%       \toprule
%       \multirow{2}{*}{\textbf{\shortstack{Select\\Freq.}}} & \multicolumn{5}{c}{\textbf{DeiT-Ti}} \\
%       \cmidrule{2-6}
%       & \textbf{L4} & \textbf{L6} & \textbf{L8} & \textbf{L10} & \textbf{L12} \\
%       \midrule
%       First-6 & \textbf{55.3} & \textbf{63.3} & \textbf{64.4} & \textbf{64.7} & \textbf{65.3} \\
%       Mid-6 & 47.4 & 62.2 & 64.3 & 64.6 & 65.2 \\
%       Last-6 & 47.1 & 56.6 & 60.3 & 61.7 & 62.3 \\
%       \bottomrule
%     \end{tabular}
%     \label{tab:freq_comparison}
%   \end{minipage}
%   \hfill
%   \begin{minipage}{0.48\textwidth}
%     \centering

%     \caption{pre-trained model's size. Pre-trained models who are more similar provide better initialization.}
%     \begin{tabular}{p{1.1cm}|ccccc}
%       \toprule
%       \multirow{2}{*}{\textbf{\shortstack{PT.\\Model}}} & \multicolumn{5}{c}{\textbf{DeiT-Ti}} \\
%       \cmidrule{2-6}
%       & \textbf{L4} & \textbf{L6} & \textbf{L8} & \textbf{L10} & \textbf{L12} \\
%       \midrule
%       DeiT-Ti & \textbf{55.3} & \textbf{63.3} & \textbf{64.4} & \textbf{64.7} & \textbf{65.3}\\
%       DeiT-S & 34.8 & 45.5 & 51.9 & 56.7 & 59.7 \\
%       DeiT-B & 29.3 & 39.2 & 47.7 & 51.8 & 57.1 \\
%       \bottomrule
%     \end{tabular}
%     \label{tab:model_comparison}
%   \end{minipage}
% \end{table}

The scale of pre-trained models significantly influences initialization effectiveness. Our results indicate that initializing with pre-trained models of similar or smaller scale produces superior outcomes. 
%Larger models contain a greater proportion of high-frequency coefficients, which are discarded during initialization, leading to increased information loss. 
Table~\ref{tab:model_comparison} presents a comparison of using DeiT-Ti, DeiT-S, and DeiT-B weights to directly initialize DeiT-Ti on ImageNet-1K for 10 epochs without additional training. Notably, even a 2.2M parameter subset extracted from the 78M parameters of DeiT-B provides effective initialization, resulting in a 2.2\% average accuracy improvement.

\subsection{Frequency-Dependent Adaptation Patterns in FRONT-initialized Models}
\label{app:freq_init}
\begin{figure}[htbp]
    \centering
    \begin{minipage}[t]{0.48\textwidth}
        \centering
        \includegraphics[width=\textwidth]{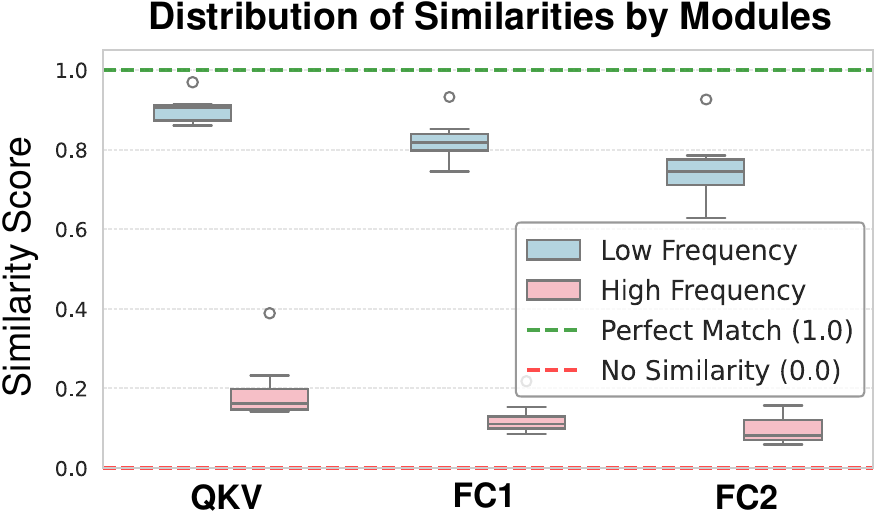}
        \caption*{(a) Distribution of Similarities by Modules}
    \end{minipage}
    \hfill
    \begin{minipage}[t]{0.48\textwidth}
        \centering
        \includegraphics[width=\textwidth]{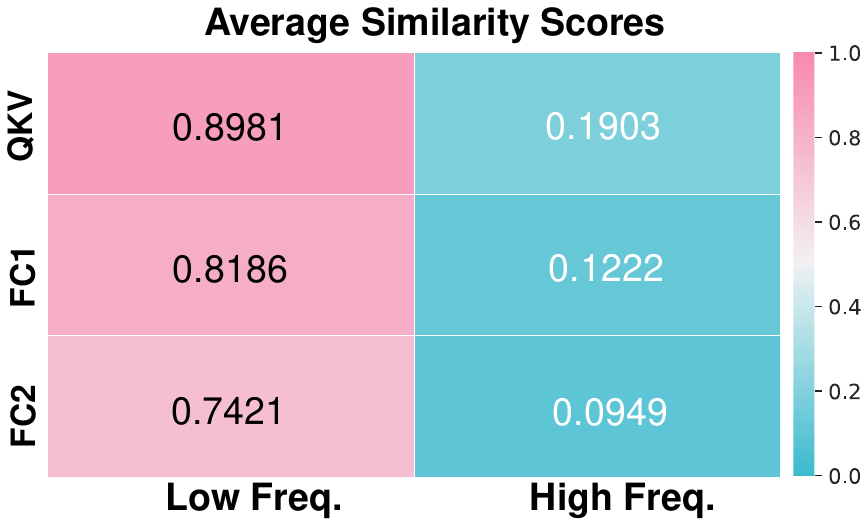}
        \caption*{(b) Average Similarity Scores}
    \end{minipage}
    \caption{Analysis of Module-wise Similarity Distribution and Average Scores for QKV, FC1, and FC2 Across Downstream Tasks}
    \label{fig:sim-scores}
\end{figure}

\noindent

Our analysis reveals distinct patterns in how FRONT-initialized models adapt to downstream tasks during fine-tuning. Figure~\ref{fig:sim-scores} presents similarity scores between weights across transformer components fine-tuned on various downstream datasets detailed in Section~\ref{app:dataset}. Figure~\ref{fig:energy_comparison} reveals that publicly available pre-trained DeiT models exhibit different frequency-domain characteristics compared to randomly initialized models. 

\begin{figure}
  \centering
  \includegraphics[width=\linewidth]{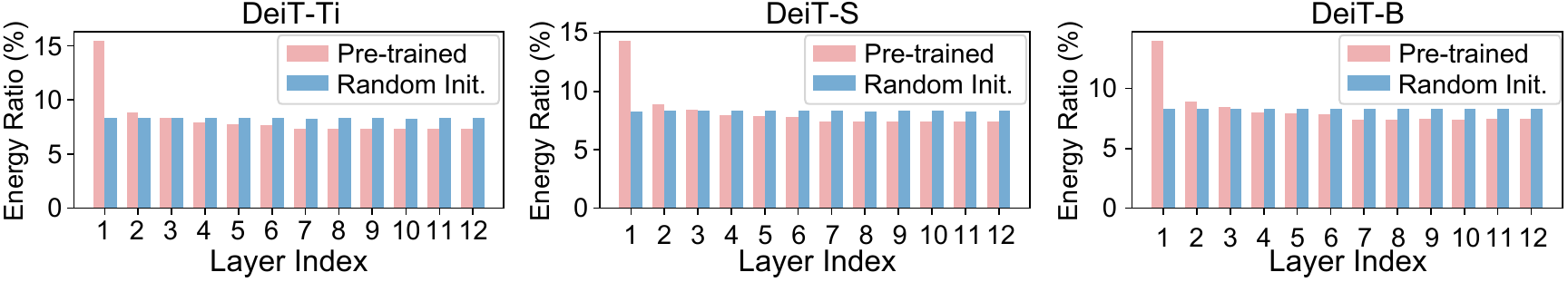}
  \vspace{-0.24in}
  \caption{Freqency domain energy distribution comparison between pre-trained and randomly initialized DeiT-Ti/DeiT-S/DeiT-B models after 3D-DCT.}
  \label{fig:energy_comparison}
  \vspace{-0.27in}
\end{figure}

Low-frequency parameter components maintain consistently high similarity scores across modules and tasks, indicating these foundational elements remain relatively effective during fine-tuning. In contrast, high-frequency components display significantly lower similarity scores, suggesting substantial updates during adaptation. This pattern supports our hypothesis that low-frequency components encode generalized linguistic knowledge transferable across tasks, while high-frequency parameters specialize to capture task-specific nuances. The model appears to selectively update high-frequency components while preserving low-frequency elements during adaptation. These findings provide valuable insights for developing more efficient fine-tuning strategies and validate the theoretical underpinnings of our approach.

\subsection{Faster convergence and stronger learning ability}

\begin{figure}

  \centering
  % {\rule[-.5cm]{0cm}{4cm} \rule[-.5cm]{4cm}{0cm}}
  \includegraphics[width=0.8\linewidth]{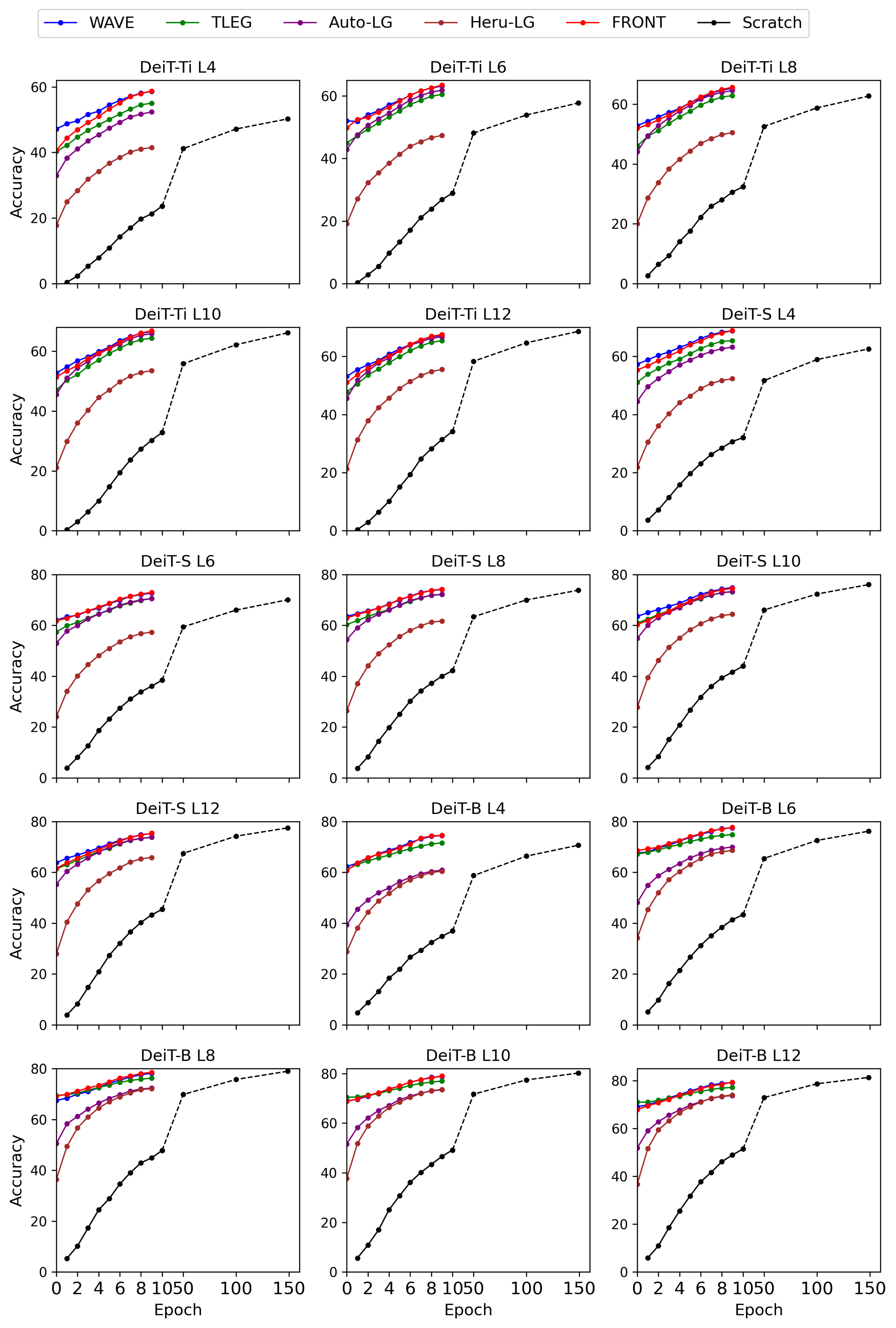}
  \caption{Performance comparisons on ImageNet-1K about depth expansion among FRONT and other learngene methods.}
  \label{fig:curve}
\end{figure}

Figure~\ref{fig:curve} demonstrates the efficacy of various learngene methods in initializing DeiT models across different scales (tiny, small, base) and layer configurations (L4-L12). Results indicate that FRONT exhibits superior initialization capabilities compared to alternative learngene approaches and performs comparably to models trained from scratch for 150 epochs.

FRONT consistently achieved higher accuracy during the initial 10 epochs across most models. Notably, FRONT does not require the introduction of stochastic parameters, suggesting its effectiveness derives from a deterministic, efficient, and comprehensive initialization strategy. All models demonstrated robust initialization performance in the first training epoch, with FRONT achieving an average accuracy of 70.08\% by epoch 7—merely 0.12 percentage points below the benchmark value of 70.20\% obtained after 150 epochs of scratch training. This initialization capability effectively compresses the model warm-up phase by approximately 21-fold, substantially reducing computational costs and highlighting FRONT’s potential to accelerate deep learning workflows. This represents a significant advantage over certain learngene methods that necessitate additional parameters or complex architectural modifications. FRONT and WAVE typically exhibit similar performance patterns, generally outperforming TLEG, Auto LG, and Heru-LG during early epochs. While Heru LG demonstrates relatively lower accuracy in initial training stages, its performance improves with continued training. As an efficient parameter initialization method, FRONT significantly accelerates downstream model training while maintaining final performance comparable to fully trained baseline models.

\section{Limitations and Future Work}

{Our method exhibits limitations when the architectural gap between source and target models is extreme. For example, transferring from a self-attention–dominant network (e.g., Transformer) to a convolution-only network (e.g., pure CNN) typically degrades performance, as these architectures lack the shared tensor-axis semantics required by our DCT-based learner representation (Assumption~1 in Appendix~\ref{sec:appendix_theory}). These cross-paradigm transfers represent known difficulties in the field and merit systematic investigation.}

Future work should focus on developing more general transformation techniques for pre-trained models. Specifically, we aim to:
\begin{itemize}
    \item Optimize extraction processes to better capture activated parameters from pre-trained weights.
    \item Design transformation approaches that maintain accuracy without additional training cost.
    \item Extend the method to handle extreme cross-paradigm scenarios by learning intermediate representations robust to axis misalignment.
    {\item Explore adaptive mechanisms for high-frequency (HF) component transfer, capable of quantifying source–target compatibility and dynamically adjusting suppression strength to maximize reusable fine-detail information.}
\end{itemize}

Our ultimate goal is to enable high-performance, training-free model adaptation at minimal computational cost, facilitating deployment of large models in resource-constrained environments and accelerating the development of efficient AI systems on diverse devices. 


% You can have as much text here as you want. The main body must be at most $8$
% pages long. For the final version, one more page can be added. If you want, you
% can use an appendix like this one.

% The $\mathtt{\backslash onecolumn}$ command above can be kept in place if you
% prefer a one-column appendix, or can be removed if you prefer a two-column
% appendix.  Apart from this possible change, the style (font size, spacing,
% margins, page numbering, etc.) should be kept the same as the main body.
%%%%%%%%%%%%%%%%%%%%%%%%%%%%%%%%%%%%%%%%%%%%%%%%%%%%%%%%%%%%%%%%%%%%%%%%%%%%%%%
%%%%%%%%%%%%%%%%%%%%%%%%%%%%%%%%%%%%%%%%%%%%%%%%%%%%%%%%%%%%%%%%%%%%%%%%%%%%%%%

\end{document}